\documentclass[twoside, 11pt]{article}

\usepackage{graphicx, subcaption}
\usepackage[margin=1in, footnotesep=0.5in]{geometry}
\usepackage[sort&compress, numbers]{natbib} 
\usepackage{amsfonts, amsmath, amssymb, amsthm}
\usepackage[shortlabels]{enumitem}
\usepackage{booktabs}
\usepackage{mathtools}
\mathtoolsset{showonlyrefs}

\usepackage{color}
\definecolor{orange}{RGB}{250,100,0}
\definecolor{customdarkgreen}{RGB}{0,150,0}
\definecolor{customdarkblue}{RGB}{0,0,150}

\usepackage[draft=false]{hyperref}
\definecolor{customdarkred}{RGB}{150,0,0}
\definecolor{customdarkgreen}{RGB}{0,150,0}
\definecolor{customdarkblue}{RGB}{0,0,150}
\hypersetup{colorlinks=true, linkcolor=customdarkred, citecolor=customdarkgreen, urlcolor=customdarkblue}
\usepackage{url}
\usepackage[linesnumbered, ruled]{algorithm2e}
\usepackage{authblk}

\def\supp{{\rm supp}}
\def\neigh{\leftrightarrow}

\newtheorem{theorem}{Theorem}[section]
\newtheorem{proposition}[theorem]{Proposition}
\newtheorem{lemma}[theorem]{Lemma}

\theoremstyle{definition}

\theoremstyle{remark}
\newtheorem{remark}[theorem]{Remark}

\numberwithin{equation}{section}
\numberwithin{figure}{section}

\def\beq{\begin{equation}} % \setcounter{equation}{1}}
\def\eeq{\end{equation}}
\def\beqn{\begin{eqnarray*}}
\def\eeqn{\end{eqnarray*}}
\def\Bitem{\begin{itemize}\setlength{\itemsep}{.2in}}
\def\bitem{\begin{itemize}\setlength{\itemsep}{.05in}}
\def\eitem{\end{itemize}}
\def\Benum{\begin{enumerate}\setlength{\itemsep}{.2in}}
\def\benum{\begin{enumerate}\setlength{\itemsep}{.05in}}
\def\eenum{\end{enumerate}}
\def\bmult{\begin{multline*}}
\def\emult{\end{multline*}}
\def\bcenter{\begin{center}}
\def\ecenter{\end{center}}
\def\bframe{\begin{frame}}
\def\eframe{\end{frame}}

\newcommand{\thmref}[1]{Theorem~\ref{thm:#1}}
\newcommand{\prpref}[1]{Proposition~\ref{prp:#1}}

\newcommand{\lemref}[1]{Lemma~\ref{lem:#1}}
\newcommand{\secref}[1]{Section~\ref{sec:#1}}
\newcommand{\figref}[1]{Figure~\ref{fig:#1}}

\DeclareMathOperator*{\argmax}{arg\, max}

\def\cC{\mathcal{C}}

\def\cE{\mathcal{E}}

\def\cG{\mathcal{G}}

\def\cS{\mathcal{S}}

\def\cU{\mathcal{U}}

\def\cY{\mathcal{Y}}

\def\bB{\mathbf{B}}

\def\bI{\mathbf{I}}

\def\bbR{\mathbb{R}}

\newcommand{\E}{\operatorname{\mathbb{E}}}
\renewcommand{\P}{\operatorname{\mathbb{P}}}

\def\Bin{\text{Bin}}

\def\eps{\epsilon}

\def\comp{\mathsf{c}}
\def\1{\mathbbm{1}}

\title{Graph Max Shift: A Hill-Climbing Method for Graph Clustering}
\author[1,2]{Ery Arias-Castro}
\author[1]{Elizabeth Coda} 
\author[3]{Wanli Qiao}
\affil[1]{\small Department of Mathematics, University of California, San Diego} 
\affil[2]{\small Halıcıoğlu Data Science Institute, University of California, San Diego}
\affil[3]{\small Department of Statistics, George Mason University}
\date{}

\begin{document}
\maketitle
\thispagestyle{empty}

\begin{abstract}
We present a method for graph clustering that is analogous to gradient ascent methods previously proposed for clustering points in space. The algorithm, which can be viewed as a max-degree hill-climbing procedure on the graph, iteratively moves each node to a neighboring node of highest degree. We show that, when applied to a random geometric graph whose nodes correspond to data drawn i.i.d. from a density with Morse regularity, the method is asymptotically consistent. Here, consistency is in the sense of Fukunaga and Hostetler, meaning, with respect to the partition of the support of the density defined by the basins of attraction of the density gradient flow.
\end{abstract}

\section{Introduction}
\label{sec:intro}

A hill-climbing algorithm is typically understood as an algorithm that makes `local' moves. In a sense, this class of procedures is the discrete analog of the class of gradient-based and higher-order methods in continuous optimization. Such algorithms have been proposed in the context of graph partitioning, sometimes as a refinement step, where the objective function is typically a notion of cut and local moves often take the form of swapping vertices in order to improve the value of the objective function. 
More specifically, consider an undirected graph consisting of $n$ nodes, which we take to be $[n] := \{1, \dots, n \}$ without loss of generality, and adjacency matrix $A = (a_{ij})$, so that $a_{ij} = 1$ when $i$ and $j$ are connected in the graph, and $a_{ij} = 0$ otherwise. Consider the simplest variant of graph partitioning, where assuming $n$ is even, the goal is to bisect the node set into two subsets, $I \subset [n]$ and $I^\comp$, in order to minimize the cut:
${\rm cut}(I) = \sum_{i \in I} \sum_{j \in I^\comp} a_{ij}.$ Arguably the simplest hill-climbing procedure would start from some partition, perhaps a random one, and then iteratively swap the pair of vertices, one in each subset of the current partition, that makes the cut decrease the most, and stop when no swap makes the cut decrease. \citet{carson2001} propose such an algorithm, which although relatively simple, is shown to be optimal for partitioning a planted bisection in the sense of achieving the information theoretic recovery threshold up to a poly-logarithmic factor.  
The procedure involves a careful initialization so that the hill-climbing part, meaning the swapping of pairs, may be seen as a refinement. This is also the case in other algorithms such as those of \cite{karypis1998fast, condon1999algorithms,  mossel2015consistency}. Sometimes the swapping is randomized \cite{jerrum1993simulated, juels1996topics, dimitriou1998go}. Hill-climbing algorithms for graph partitioning date back at least to the work of \citet{kernighan1970efficient}, who propose a more sophisticated swapping protocol than the one described above. The partition around medoids (PAM) algorithm of \citet{kaufmann1987clustering} includes a hill-climbing component.
%implementation/numerical
%\cite{fiduccia1988linear}
%\cite{johnson1989optimization}
%\cite{yeh1995optimization}

In the present paper, we attach a totally different meaning to `hill-climbing' that comes from the problem of clustering points in a Euclidean space as proposed by \citet{fukunaga1975}. (We note that this definition has been shown in \cite{arias2023b,arias2023} to be intimately related to the level set definition of clustering proposed by \citet{hartigan1975}.) 
Assuming the points were generated iid from some density which is regular enough for the associated gradient ascent flow to be well-defined, the space, in fact the density support, is partitioned according to the basins of attraction of that gradient flow. 
The hill-climbing characterization comes from being a gradient method: we move each point in the density support upward, in the landscape given by the density, following the steepest ascent, and points that end at the same location are grouped together. Of course, in practice the density is unknown and one needs to find workarounds. The most famous one is the mean-shift algorithm \cite{fukunaga1975, cheng1995}, although a number of other proposals have been suggested, including the most straightforward one which consists in first estimating the density, e.g., by kernel density estimation, and then computing the basin of attraction of the estimated density using a Euler scheme or similar --- see \cite{arias2025clustering}, where a number of these procedures are analyzed. 

The bridge between the discrete setting of a graph and the continuous setting of a Euclidean space is the usual one: geometric graphs, meaning, graphs whose node set is embedded in a Euclidean space. Our focus will be random geometric graphs with nodes corresponding to points that are generated iid from a density with sufficient regularity, and edges connecting all pairs of nodes whose associated points are within distance $\eps$. 
The basic question that motivates the present work is the following:
\begin{quote}
\textit{Given the adjacency matrix of a random geometric graph, is it possible to partition the graph in a way which is consistent with the clustering given by the gradient ascent flow of the density?}
\end{quote}
We emphasize that the geometric graph is only known through its adjacency matrix. No embedding of the graph is provided, in other words, the points associated with the nodes are unknown. (In this context, these points are often referred to as `latent positions'.)

% \textit{Given a random geometric graph with nodes corresponding to vertices that are generated iid from a density with sufficient regularity, is it possible to partition the graph in a way which is consistent with the clustering given by gradient ascent flow of the density?}

We answer this question affirmatively in the following sense: We propose an algorithm for graph partitioning which, when applied to a random geometric graph with underlying points generated iid from a density with enough regularity, outputs a partition which is consistent with the partition given by the basins of attraction of the gradient ascent flow of the density in an asymptotic setting, driven by $n\to\infty$, where the connectivity radius of the graph tends to zero slowly enough. Importantly, the algorithm we propose does not attempt to embed the graph. Although a strategy that would consist in embedding the graph and then applying a method for gradient ascent flow clustering is tempting, the problem of embedding the graph is arguably harder than the partitioning problem that we consider here.

For more connections to the literature, we can single out the algorithm of \citet{koontz1976}, which although proposed to cluster points in space, at its core is a hill-climbing graph partition algorithm close in intention to ours --- but as pointed out in \cite{arias2025clustering}, where it is called Max Slope Shift, the algorithm in its original form fails to achieve its purpose. Our algorithm can also be seen as a special case of a class of algorithms proposed by \citet{strazzeri2022}, although the intention there is very different. We elaborate on these works later on in \secref{connections}. 

The rest of the article is organized as follows. 
In \secref{prelim} we provide some background on gradient ascent flow clustering. 
In \secref{algorithm} we describe the algorithm we propose, which we call Graph Max Shift, make some connections to other algorithms, and establish its asymptotic consistency. 
In \secref{numerics} we present a series of numerical simulations demonstrating the use of the method. (Code is available to reproduce all experiments at \url{https://github.com/lizzycoda/GraphMaxShift}.) 
We conclude with a brief discussion in \secref{discussion}.

\section{Preliminaries}
\label{sec:prelim}

In the gradient flow approach to clustering, which originated in the work of \citet{fukunaga1975}, each $x \in \supp(f)$ is assigned to the limit point following along the direction of the gradient of $f$. Formally, assume that $\nabla f$ is Lipschitz so that the gradient ascent flow exists. Then the gradient flow line originating at a point $x$ is the curve $\gamma_x$ defined by
\begin{align}
\label{grad_flow_defn}
\gamma_x(0) = x; \qquad \dot{\gamma}_x(t) = \nabla f(\gamma_x(t)), \quad t\ge 0.
\end{align}
The basin of attraction of a point $x^*$ is $\{ x: \gamma_x(\infty) = x^* \}$. Note that the basin of attraction of $x^*$ is empty unless $x^*$ is a critical point. It turns out that, if $f$ is a Morse function \citep{milnor1963morse}, meaning that it is twice differentiable and its Hessian is non-singular at every critical point, then the basins of attraction corresponding to local maxima, sometimes called modes, provide a partition of the support of $f$ up to a set of zero measure. See, e.g., \cite[Lem 2.2]{arias2025clustering}, which is argued entirely based on \cite[Cor 3.3, Th 4.2]{banyaga2013lectures}.

A number of mode-seeking algorithms have been proposed to approximate the gradient flow lines \cite{carreira2015}, some of them analyzed and proved to be consistent in \cite{arias2025clustering}. In the same paper, another algorithm was proposed, called Max Shift, which consists in iteratively moving to the point with highest (estimated) density value. 
With a search radius $r > 0$, which is the only tuning parameter of the method, starting at a given point a sequence is constructed by iteratively including a point within distance $r$ with highest density value: if the starting point is $x_0$, then Max Shift computes the sequence $(x_k)$ given by
\begin{align}
\label{max_shift_r}
x_{k+1} \in  \argmax\big\{f(x) : x \in \bar{B}(x_k, r)\big\}, \quad k\ge 0,
\end{align}
where $\bar B(y, r)$ is the closed ball centered at $y$ of radius $r$.
Ties are broken in some arbitrary but deterministic way, e.g., by following the lexicographic order. The algorithm is hill-climbing in the sense that at each step it `climbs' the landscape given by the density. Max Shift clusters together points whose sequence have the same endpoint. 

In practice, the density $f$ is unknown and the straightforward strategy is to apply the procedure to an estimate of the density. In detail, given data points $\cY = \{y_1, \dots,y_n\}$ assumed to be iid from $f$, the Max Shift method utilizes a kernel density estimator $\hat f_\eps$ with bandwidth $\eps$ to estimate the density and proceeds as before but with $\hat f_\eps$ in place of $f$. In the so-called medoid variant, only data points are considered, resulting in the following procedure: starting at some data point $y_{i_0} \in \cY$, 
\begin{align}
\label{max_shift_method_r}
y_{i_{k+1}} \in  \argmax\big\{\hat f_\eps(y) : y \in \cY \cap \bar{B}(y_{i_k}, r)\big\}, \quad k\ge 0.
\end{align}
For references on medoid variants of various hill-climbing clustering methods, see \cite[Appendix B]{arias2025clustering}. 

Max Shift and its medoid variant are shown in \cite{arias2025clustering} to be asymptotically consistent in the sense of \citet{fukunaga1975} (formally defined in \secref{consistency}) under some conditions on the underlying density $f$ and on the kernel and bandwidth defining the density estimate. Specifically, it is assumed the density converges to zero at infinity; that it is twice differentiable with uniformly continuous zeroth, first, and second derivatives; and that it is of Morse regularity; and it is assumed that the kernel used is of second order and the bandwidth is chosen so that the density estimator is second order consistent, meaning that $\eps = \eps_n \to 0$ in such a way that, as $n\to\infty$, 
\begin{equation}
\begin{gathered}
\sup_x |\hat f_\eps(x) - f(x)| \to 0, \qquad
\sup_x \|\nabla \hat f_\eps(x) - \nabla f(x)\| \to 0, \qquad 
\sup_x \|\nabla^2 \hat f_\eps(x) - \nabla^2 f(x)\| \to 0.
\end{gathered}
\end{equation}

Our proposed algorithm for graph clustering, introduced later in \secref{algorithm}, coincides with Max Shift (based on a particular kernel density estimator) in the case of a geometric graph. However, the consistency of Max Shift established in \cite{arias2025clustering} does not extend to the present setting, as the kernel function that Max Shift implicitly uses does not satisfy the conditions required \cite{arias2025clustering}. 
%We make this point again in \remref{flat}.

\section{Graph Max Shift}
\label{sec:algorithm}

This is the section where we introduce our algorithm for graph partitioning that, when the graph under consideration happens to be a random geometric graph, implements a form of gradient ascent flow clustering of the underlying positions. 

\subsection{Description} 

We describe the graph clustering method for a general graph with $n$ nodes, taken to be $[n] := \{1, \dots, n\}$ without loss of generality, and adjacency matrix $A = (a_{ij})$ where, as usual, $a_{ij} = 1$ exactly when $i$ and $j$ are neighbors in the graph, and $=0$ otherwise. We will use the convention that $a_{ii} = 1$ for all $i$, so that each node is a neighbor to itself. We will sometimes use the notation $i \neigh j$ to indicate that $i$ and $j$ are neighbors, i.e., $a_{ij} = 1$.
The graph could be, but is not necessarily a geometric graph. 
The degree function is defined as usual, $q_i = \# \{j : j \neigh i \}$.
%For an unweighted graph, $h = 1$ is the only nontrivial choice, and in that case $q_i$ is simply the degree of node $i$ in the graph. 

Graph Max Shift, initialized at node $i$, computes the following sequence of nodes:
\begin{equation}
\label{graphmaxshift}
i_0 = i; \qquad i_t \in \argmax \{q_j : j \neigh i_{t-1}\}, \quad t\ge 1.
\end{equation}
Ties are broken in some arbitrary, but deterministic way, e.g., by choosing the node with highest numerical index among maximizers. 
We call the path $(i_t)$ the hill-climbing path originating from $i$ where the landscape is given by the degree function $q$. In plain words, Graph Max Shift iteratively moves to a neighbor with highest degree. 
We then cluster together nodes whose hill-climbing paths end at the same terminal node. Optionally, we merge together any two clusters whose associated terminal nodes are within $\tau$ hops --- in which case $\tau \ge 1$ is a tuning parameter of the method (the only one).

% As we will show in proof and numerical experiments, generally setting $r=h$ is fine, though in some cases if may be more efficient to use different values. 

\subsection{Connections to other algorithms}
\label{sec:connections}

\subsubsection{Connection to Max Shift}
\label{sec:max_shift}

Our claim is that Graph Max Shift implements gradient ascent flow clustering when applied to a geometric graph. Let $\cY := \{y_1, \dots, y_n\}$ be a set of points in some Euclidean space, say, $\bbR^d$, and consider the neighborhood graph with connectivity radius $\eps$, denoted $\cG(\cY ; \eps)$, defined by the adjacency matrix $a_{ij} = 1$ if $\|y_i - y_j\| \le  \eps$, and $= 0$ otherwise.  

It is straightforward to check that Graph Max Shift applied to $\cG(\cY ; \eps)$ will compute the same hill-climbing paths as Max Shift \eqref{max_shift_method_r} applied to $\cY$ when using the so-called flat kernel
\begin{equation}
\label{kernel}
K(x) = v_d^{-1} \bI(\|x\|\le 1),
\end{equation}
where $v_d$ is the volume of the unit ball in $\bbR^d$, and when the bandwidth and search radius are taken to be $\eps$.
Indeed, with that kernel function and bandwidth, the density estimator used by Max Shift is given by
\begin{equation}
\label{kde}
\hat f_\eps(x) = \frac{1}{n \eps^d} \sum_{i=1}^n K\Big(\frac{x-y_i}{\eps}\Big) = \frac{\#\{i : \|x-y_i\| \le \eps\}}{v_d n \eps^d},
\end{equation}
so that, when applied to a data point, 
\[
\hat f_\eps(y_i) = \frac{q_i}{v_d n \eps^d} \propto q_i.
\]
When started at $y^0 \in \cY$, the sequence built by Graph Max Shift, from the point of view of the latent positions, is given by
\begin{align}
\label{graph_max_shift}
y^{k+1} \in  \argmax\limits_{y \in \cY \cap \bar{B}(y^k, \eps)} \hat f_\eps(y), \quad k\ge 0,
\end{align}
which is exactly the sequence computed by Max Shift applied to $\cY$ based on the kernel density estimator $\hat f_\eps$.
(Note that $y^k$ above corresponds to $y_{i_k}$ for some $i_k \in [n]$.)

\begin{remark}
We remark that, in the situation where the user is only provided with an adjacency matrix, which is the one we envision, when applying Graph Max Shift the user does not have a choice of kernel function $K$ or bandwidth $\eps$. 
\end{remark}

\subsubsection{Connection to Max Slope Shift}
\label{sec:koontz}
\citet{koontz1976} proposed an algorithm for clustering points in space, which is called Max Slope Shift in \cite{arias2025clustering} for reasons that will become clear shortly. Given a point set $\cY := \{y_1, \dots, y_n\}$, it forms the weighted neighborhood graph with connectivity radius $\eps$, namely, the graph on $[n]$ with dissimilarity between nodes $i, j \in [n]$ given by $\delta_{ij} = \|y_i - y_j\|$ if $\|y_i-y_j\| \le \eps$, and $\delta_{ij} = \infty$ otherwise. 
Here $i$ and $j$ are considered neighbors if $\delta_{ij} < \infty$, and the degree function is defined as before, $q_i = \#\{j : j \neigh i\}$.
Then, starting at node $i$, Max Slope Shift computes the following sequence of nodes:
\begin{equation}
\label{maxslopeshift}
i_0 = i; \qquad i_t \in \argmax \left\{\frac{q_j - q_{i_{t-1}}}{\|y_j - y_{i_{t-1}}\|} : j \neigh i_{t-1}\right\}, \quad t\ge 1.
\end{equation}

As argued in \cite{arias2025clustering}, it is not hard to see that the end points will be local maxima not of the density itself, but of its gradient, meaning, density inflection points. The algorithm thus fails to output an accurate clustering as defined by \citet{fukunaga1975}. (A regularization is proposed in \cite{arias2025clustering} to remedy this behavior.)

Max Slope Shift can be nonetheless seen to be closely related to ours. Indeed, it implies the use of the flat kernel for density estimation; and, if one were to imagine how it would be adapted to clustering an unweighted graph, one would end up with Graph Max Shift, as the natural thing to do in that case is to replace the Euclidean distance in the denominator in the fraction appearing in \eqref{maxslopeshift} with~1.

\subsubsection{Connection to Morse clustering}
\label{sec:strazzeri}
Graph Max Shift can be seen as a special case of a class of clustering algorithms called Morse clustering recently proposed by \citet{strazzeri2022}. 
A Morse clustering algorithm is based on an edge preorder and a node preorder, which we denote by $\preceq_{E}$ and $\preceq_{N}$, respectively. The choice of these preorders defines the algorithm, which proceeds as follows: Starting at $i_0 = i$, if the path is at $i_{t-1}$ at step $t-1$, the path ends there unless 1) there is a node $j \neigh i_{t-1}$ such that $(i_{t-1},k) \prec_{E} (i_{t-1},j)$ for all $k \neigh i_{t-1}, k \ne j$; and 2) $i_{t-1} \prec_{N} j$; if conditions 1) and 2) are met, then the node $j$ is unique, and the process continues with $i_t = j$. 
Doing this for all nodes $i$ defines what \citet{strazzeri2022} call the Morse flow.
Nodes whose flows end at the same terminal node are clustered together.

Graph Max Shift is a Morse clustering algorithm. Indeed, for convenience, assume that Graph Max Shift breaks ties between nodes by selecting the largest one when nodes are represented by their indices. Then define the node preorder $i \preceq_{N} j$ if $i=j$, or $q_i = q_j$ and $i < j$, or $q_i < q_j$; and define the edge preorder $(i,j) \preceq_{E} (i,k)$ if $j \preceq_{N} k$. Then it is a matter of checking that the Morse clustering algorithm with these two preorders coincides with Graph Max Shift.

Although our algorithm belongs to the class of Morse clustering algorithms, our motivation and that of \citet{strazzeri2022} are very different. Indeed, the impetus in \cite{strazzeri2022} for proposing Morse clustering algorithms is to address the impossibility result of \citet{kleinberg2002}. In the context of weighted graphs, Kleinberg proposes three axioms that a clustering algorithm should satisfy, {\em scale-invariance}, {\em richness}, and {\em consistency}, and then proceeds to show that no algorithm satisfies all three. 
\citet{strazzeri2022} offer some criticism of the consistency axiom, propose to replace it by a weaker axiom which they call {\em monotonic consistency}, and show that there is at least one Morse clustering algorithm which satisfies this new axiom, and the other two axioms of scale-invariance and richness.   

Due to the fact that the intention is very different to ours in the present work, it is not surprising that there is no analysis in \cite{strazzeri2022} of asymptotic consistency in the sense considered here, i.e., the recovery of gradient ascent flow clustering in the context of a random geometric graph. 
 
\subsection{Consistency}
\label{sec:consistency}  

In this section, we establish the asymptotic (statistical) consistency of Graph Max Shift when applied to $\cG(\cY, \eps_n)$, when $\cY := \{y_1, \dots, y_n\}$ is generated iid from a density $f$ on $\bbR^d$ which satisfies some regularity conditions, and when $\eps_n \to 0$ slowly enough. In more detail, the reference or population partition is the one defined by \citet{fukunaga1975}, namely, $i$ and $j$ are in the same population cluster if $y_i$ and $y_j$ belong to the basin of attraction of the same mode of $f$. 
Consider the distance between two partitions of $[n]$ defined as the fraction of pairs of nodes $i \ne j$ that are grouped together in one partition but not in the other partition. This distance is equal to one minus the Rand index~\cite{rand1971objective} --- a measure of similarity between partitions popular in the clustering literature.
We say that Graph Max Shift, or any other clustering algorithm for that matter, is consistent if the distance between the partition it outputs and the reference partition tends to zero in probability as $n\to\infty$, or equivalently, if the Rand index tends to one. 
%More loosely, a clustering method is consistent if the fraction of nodes that are incorrectly clustered converges to zero as the sample size increases. 
By contrast, a clustering method is weakly consistent if the fraction of pairs of points that belong to different population clusters and are clustered together by the method tends to zero in probability. 
In essence, a method is consistent if it recovers the population partition asymptotically, while a method is weakly consistent if it recovers a possibly finer partition asymptotically.

More formally, and for future reference in the proof section, let $(\bB_k)$ denote the basins of attraction of density modes. (Note that there could well be infinitely many of them, although only countably many under our assumptions on the density.) Given $\cY$, the reference clusters are $\bI_k = \{i: y_i \in \bB_k\}$. Under our assumptions on the density, with probability one, $[n] = \bigcup_k \bI_k$, and this is the reference partition of the graph in the present context. 
A cluster method takes the graph adjacency matrix $A$ and returns a partition of $[n]$, say $[n] = \bigcup_k \widehat\bI_k$. 
We write $\bI(i,j) = 1$ if there is $k$ such that $i,j \in \bI_k$, and $\bI(i,j) = 0$ otherwise, and $\widehat\bI(i,j)$ has a similar meaning.
Under the model of a random geometric graph described above, the method is weakly consistent if 
\begin{align}\label{weakly consistent}
\frac{\#\big\{i < j: \widehat\bI(i,j) < \bI(i,j) \big\}}{\binom{n}{2}} \longrightarrow 0,\ \text{in probability as $n \to \infty$}; 
\end{align}
and it is consistent if
\begin{align}\label{consistent}
\frac{\#\big\{i < j: \widehat\bI(i,j) \ne \bI(i,j) \big\}}{\binom{n}{2}} \longrightarrow 0,\ \text{in probability as $n \to \infty$}. 
\end{align}
(The proportion in \eqref{consistent} is the distance between the partitions $(\bI_k)$ and $(\widehat\bI_k)$.) 

\begin{remark}
\label{rem:flat}
The consistency of Graph Max Shift is not obtained as a special case of the consistency results established in \cite{arias2025clustering}, which crucially depend on the use of a second order kernel. The simplicity of Max Shift allows us to simplify the arguments in \cite{arias2025clustering} (which were developed for generic hill-climbing methods) and, in particular, avoid the more complex aspects having to do with the stability of the basins of attraction to perturbations of the density.
\end{remark}

\begin{theorem} 
\label{thm:consistency}
Suppose $\cY_n := \{y_1, \dots, y_n\}$ is generated iid from a density $f$ on $\bbR^d$ with compact support, twice continuously differentiable, and of Morse regularity, and consider Graph Max Shift applied to the adjacency matrix of $\cG(\cY_n, \eps_n)$  with $\eps_n \to 0$. 
\begin{enumerate}[itemsep=0in]
\item[(i)] {\em Without merging.}\  
If $\eps_n \gg \big(\frac{\log n}n\big)^{1/(d+2)}$, the method is weakly consistent.
\item[(ii)] {\em With merging.}\ 
There is $C = C(f)$ such that, if $\tau \ge 1/C$ and $\eps_n \gg \big(\frac{\log n}n\big)^{1/\max\{d+4,\, 2d\}}$, the method is consistent.
\end{enumerate}
\end{theorem} 

The proof of \thmref{consistency} is in \secref{proofs}. We note that the assumption of compact support is for simplicity, as we anticipate the consistency result extends to densities that simply converge to zero at infinity as long as $f$ and its first and second derivatives are bounded --- which is the setting considered in \cite{arias2025clustering}. 
The conditions on $\eps_n$ are not particularly intuitive and arise in the course of the proof. In particular, ignoring the logarithmic factor, we do not know if the exponent is optimal in either {\em (i)}   or {\em (ii)} --- although this sort of `curse-of-dimensionality' dependence on the dimension is to be expected in the nonparametric setting that we consider given where only smoothness assumptions are made on the density.

% We note that the choice of tuning parameter $\tau = 1$ --- meaning that clusters whose representative nodes are direct neighbors are merged --- yields a consistent algorithm in the setting of the theorem when $\eps_n \to 0$ slowly enough that $n \eps_n^{d+4}/\log n \to \infty$. 

The lower bound on the number of hops $\tau$ for merging in post-processing depends on $f$ in particular through the distance between its modes. 
We note that in practice, the choice of tuning parameter $\tau = 1$ --- meaning that clusters whose representative nodes are direct neighbors are merged --- is often sufficient for good clustering performance, though without the merging post-processing step, it is possible that some of the population clusters defined by the density gradient ascent basins of attraction could be split into multiple clusters. 

\section{Numerical Experiments}
\label{sec:numerics}

In this section, we demonstrate the behavior of Graph Max Shift on random geometric graphs with different underlying densities. The implementation is built upon the igraph package \cite{csardi2006}. Python scripts to reproduce all figures in this section are available at \url{https://github.com/lizzycoda/GraphMaxShift}. 
%In these results, if there is a tie between nodes $i$ and $j$, we select node $\min\{i,j\}$. 

It should be noted that these numerical experiments are only meant to provide some confirmation for our theory. Graph Max Shift is not intended to improve the state-of-the-art, and in particular we did not compare with any other algorithm. (In fact, as we tried to make clear in the Introduction, we do not know of an algorithm that was proposed for the same purpose, so that there is no clear competitor.)

In \figref{mixtures} we demonstrate the consistency of Graph Max Shift on various (realizations of) random geometric graphs, where the underlying density $f$ is a Gaussian mixture with different numbers of components. We use the same mixture distributions as used in \cite{chacon2015}. For each distribution, we sample $n = 10^4$ data points, and apply the Graph Max Shift method without merging ($\tau = 1$). The value of $\eps$ used to construct $\cG(\cY ; \eps)$ is indicated in each figure, and was chosen to obtain a low clustering error. The solid black lines represent the borders between the clusters defined by the gradient ascent flow and were found using the numerical method described in Section 3.1 of \cite{chacon2015}. The method is based on Theorem 1 in \cite{ray2005}. For all distributions, most points are clustered together correctly, with some errors near the borders of the different basins of attractions or in low density regions. 
%This behavior is accordant with \thmref{consistency}, where we showed that with large enough sample size and appropriate choices of parameters, the clustering will be correct for points in an upper level set bounded away from the borders. 

\begin{figure}[h]
    \centering
\includegraphics[width=0.6\linewidth]{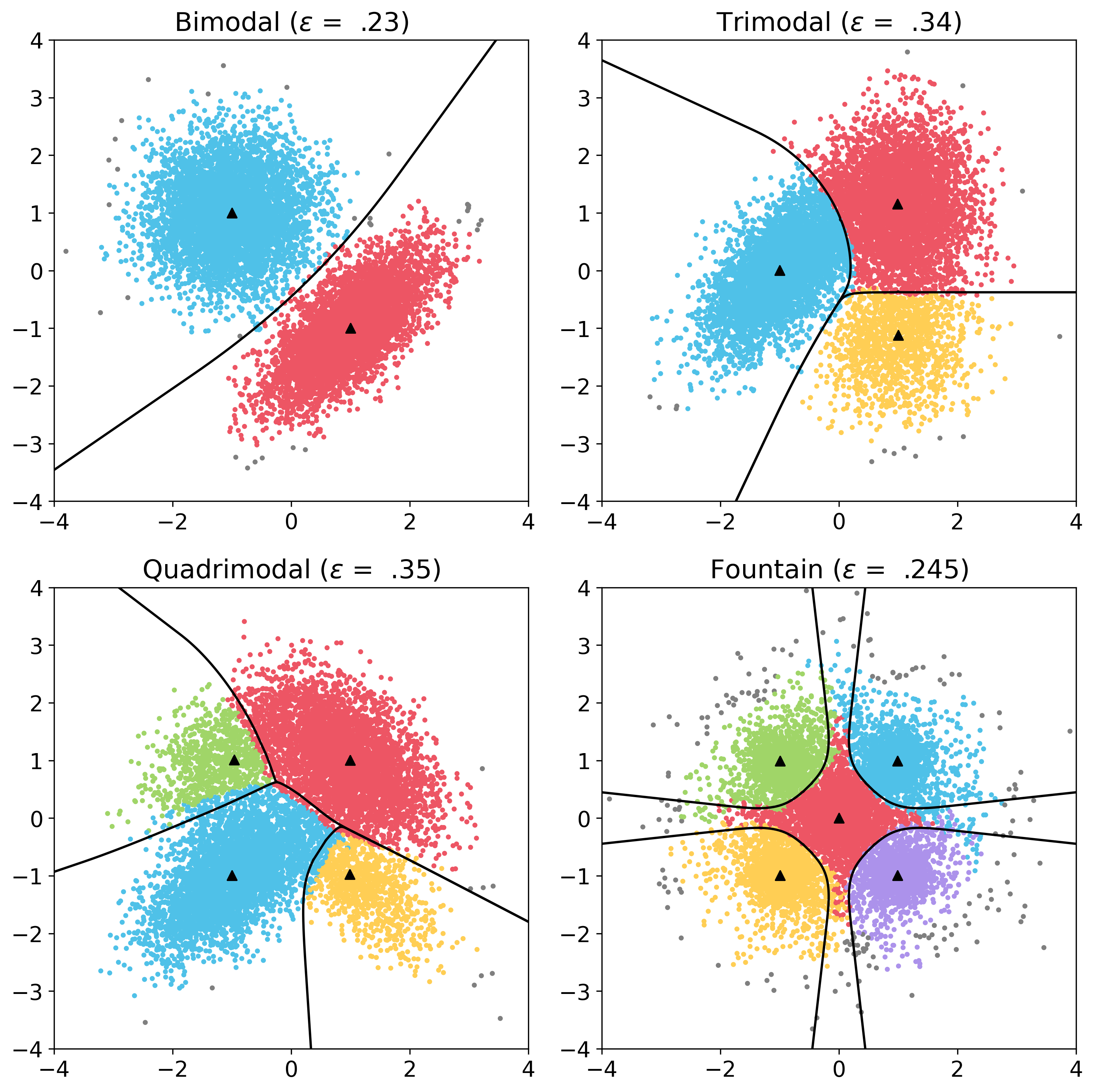}
    \caption{Graph Max Shift applied to $\cG(\cY; \eps)$ where the sample, of size $n = 10^4$, is drawn from different normal mixtures. The solid black lines represent borders between the basins of attraction (population clusters) and the triangles represent the locations of the modes of the density. As the means of the mixture components all either coincide with the modes or are very close to the modes, we do not plot them. Data points are colored according to the clustering obtained via Graph Max Shift.  Points that belong to clusters with fewer than $25$ points are colored in gray.}
    \label{fig:mixtures}
\end{figure}

\figref{tuning} depicts the qualitative effects of adjusting $\eps$ in $\cG(\cY; \eps)$ holding the tuning parameter $\tau$ fixed. When $\eps$ is too small, the density estimator is not very smooth, which results in too many modes. In the clustering step, the small search radius obtains oversegmented clusters. When $\eps$ is large, the density estimator is too smooth in the sense that some modes of the underlying density are not recovered, and coupled with the large search radius that allows points to `cross' the basins of attraction, the result is an inaccurate clustering.

In \figref{tuning_quant} we quantify the effects of adjusting $\eps$, and plot the weak clustering error as defined in \eqref{weakly consistent} and the clustering error as defined in \eqref{consistent}. In the lefthand plot, the weak clustering error remains small over a range of $\eps$ small enough. Note that when $\eps$ is smaller than the minimum distance between two points, each vertex will correspond to its own cluster and the weak clustering error will be zero. In contrast, the condition on $\eps$ in  \thmref{consistency} is imposed to ensure that points in the basin of attraction of a particular mode are moved by Graph Max Shift to a point close to that mode. For example, in \figref{tuning} when $\eps = 0.18$ or $\eps = 0.2$ the weak clustering error is small, though not all terminal nodes are close to the modes of the density. When $\eps = 0.3$, the weak clustering error remains small, and the terminal nodes are all close to the modes of the density. In the righthand plot of  \figref{tuning_quant}, the clustering error is small for a range of $\eps$ not too small and not too large. This reflects the oversegmentation when $\eps$ is too small and the crossing of the basins of attraction when $\eps$ is too large. Moreover, as $n$ increases, the values of $\eps$ for which the error is small moves to the left, reflecting the condition in \thmref{consistency}.  

\begin{figure}[h]
    \centering
\includegraphics[width=1\linewidth]{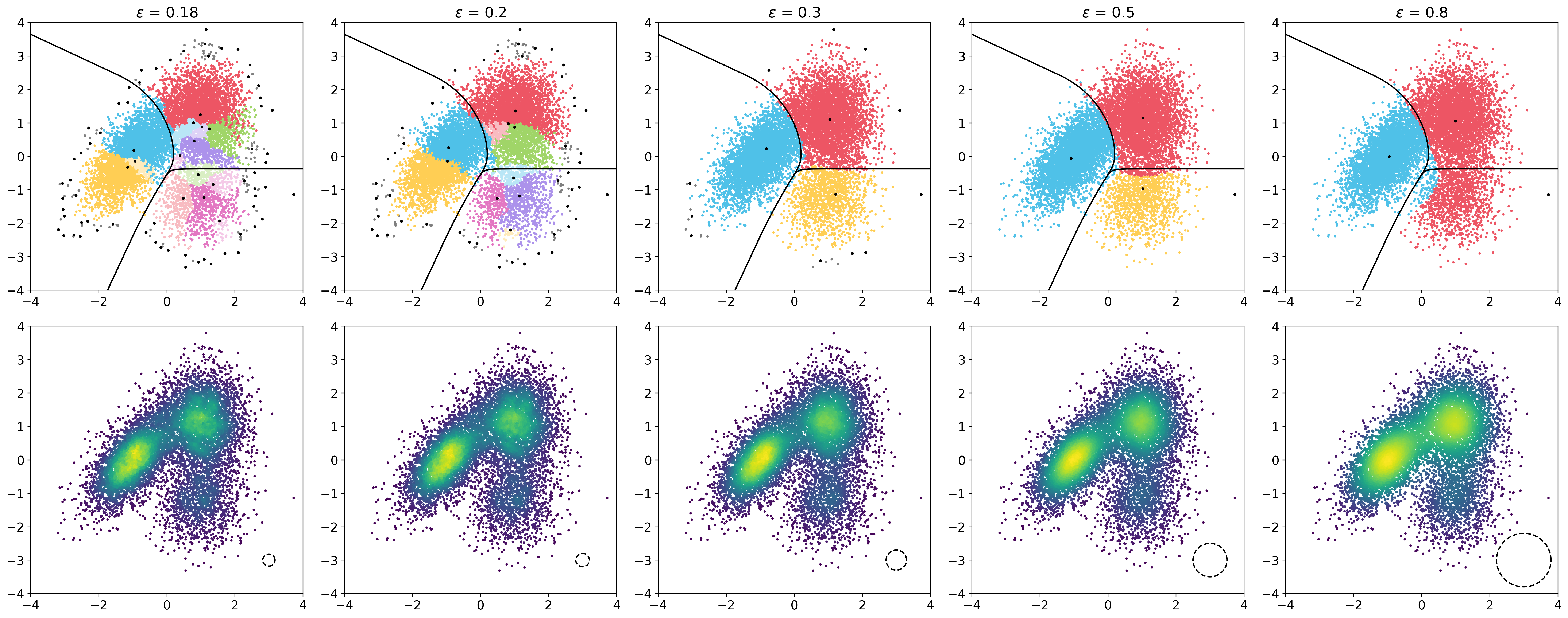}
    \caption{Graph Max Shift applied to $\cG(\cY;\eps)$ with data drawn from a trimodal Gaussian mixture. The top row shows the obtained clustering with the indicated $\eps$ and $\tau = 1$. Terminal nodes are highlighted in black. The bottom row depicts the degree of each node, which is proportional to the value of the density estimator implicitly computed. Additionally, each plot includes a ball of radius $\eps$ in the bottom right for reference. The errors of each clustering are quantified in \figref{tuning_quant}}
    \label{fig:tuning}
\end{figure}

\begin{figure}[h]
    \centering
\includegraphics[width=.8\linewidth]{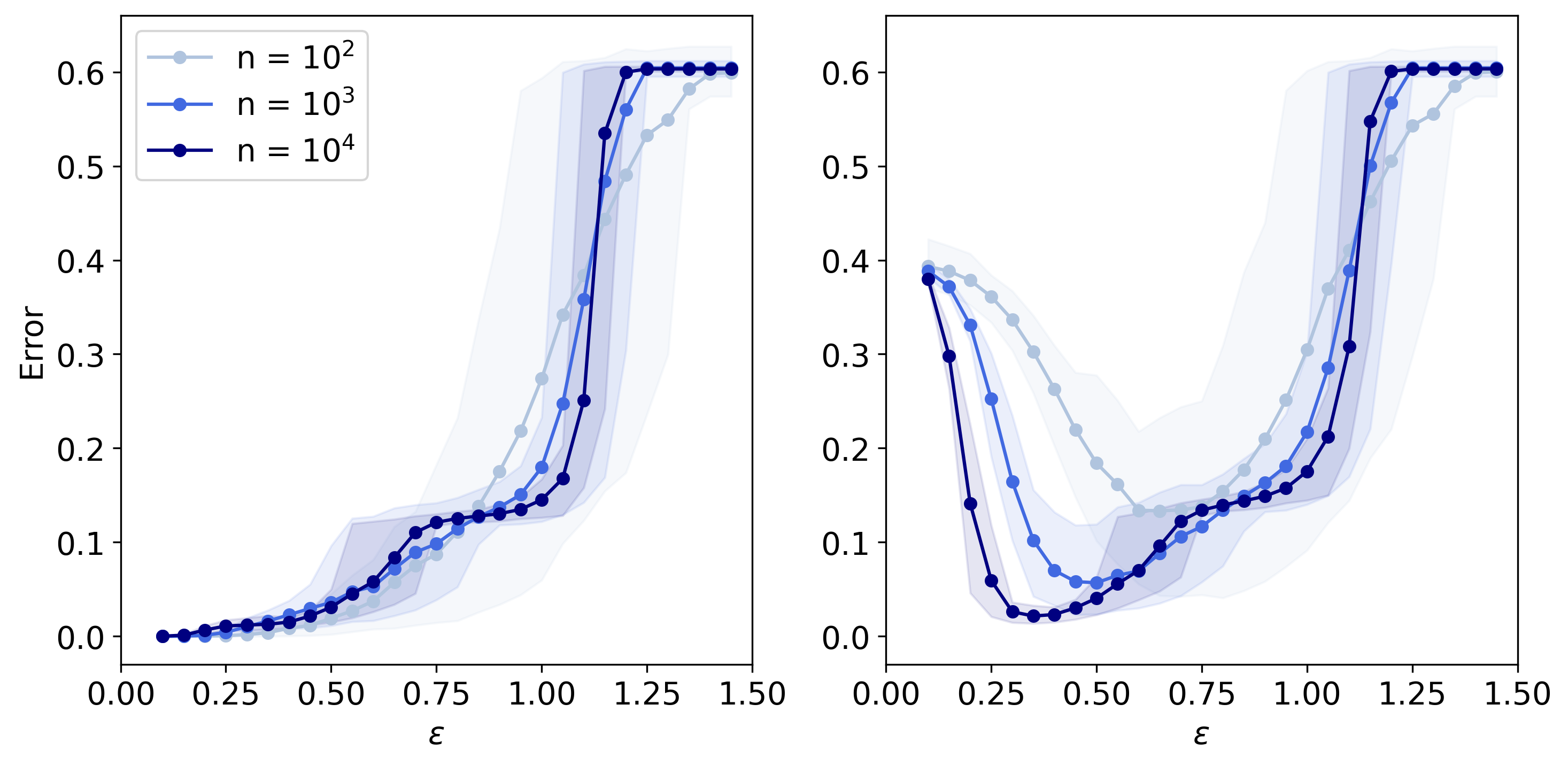}
    \caption{The error as a function of $\eps$ when Graph Max Shift is applied to $\cG(\cY;\eps)$ with $\tau = 1$ for data drawn from the trimodal Gaussian mixture for different values of $n$. The left plot shows the weak clustering error \eqref{weakly consistent} and the right plot shows the clustering error \eqref{consistent}. Solid lines indicate the  error averaged over $100$ simulations, and the shaded regions indicate empirical 80\% intervals. Clusterings for select values of $\eps$ when $n=10^4$ are shown in \figref{tuning}.}
    \label{fig:tuning_quant}
\end{figure}

Finally, in \figref{paths} we plot some paths constructed by Graph Max Shift for the trimodal mixture corresponding to $\eps = 0.4$ and $\tau = 1$. We also include an approximation of the Max Shift sequence $(x_k)$ based on $f$ itself. The proof of the main theorem relies on the sequences closely following one another, and indeed we observe this and see that they end close to the associated mode.

\begin{figure}[h]
    \centering
    \includegraphics[width=0.5\linewidth]{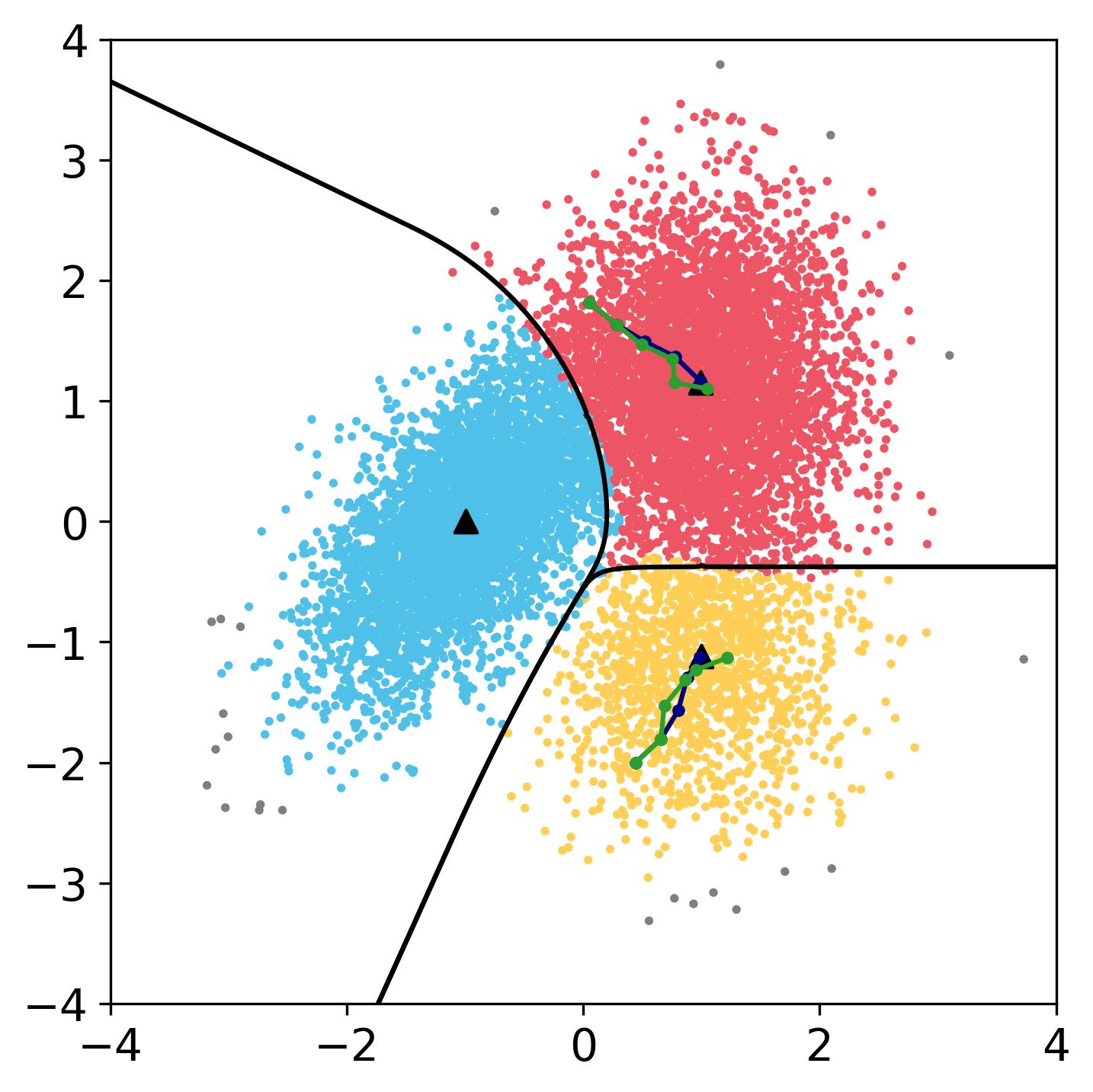}
    \caption{Graph Max Shift path $(y^k)$ in green and an approximation of the Max Shift path $(x_k)$ based on $f$ itself in navy blue. Modes of the density are plotted as triangles.}
    \label{fig:paths}
\end{figure}

We find that for some densities, even when the data is just two dimensional,  Graph Max Shift can have a large variance, even with sample sizes as large as $n \approx 10^5$. Additional quantitative plots for the other Gaussian mixtures in \figref{mixtures} are included in the Appendix which also depict this variance. We also include an additional example where the clustering task is more difficult and adjusting the tuning parameter is necessary to achieve small clustering error. 

%Memory constraints limited our running experiments on larger graphs.
% We include additional plots in Appendix~\ref{sec:additional_numerics} where $n$ is varied and the graph is not a random geometric graph. \ery{remove appendix and refer to github for additional numerical experiments?}

\section{Discussion}
\label{sec:discussion}

%\subsection{Merging parameter}
%We could not discard the presence of multiple points of local maximum degree in the vicinity of the same density mode (as specified by \lemref{local_convergence}), so that, without merging in post-processing, Graph Max Shift is only guaranteed (in the context of \thmref{consistency}) to be weakly consistent. 
%To enable consistency, we opted for a post-processing step where clusters with nearby representative nodes are merged together. Even though less than elegant, this sort of step is not unusual in the literature. The choice of this tuning parameter is not straightforward, as is the case of most tuning parameters of clustering (and other unsupervised learning) methods, but is fairly intuitive here: set the parameter $\tau$ to less than half the minimum distance between two cluster centers that the user wants to be able to distinguish. Side information specific to an application could provide some insight into that. (But, again, we do not want to emphasize the practical nature of the proposed method.) 

\subsection{Multi-hop variant}
\label{sec:multi-hop}

When Graph Max Shift as presented in \secref{algorithm} is applied to a geometric graph, the connectivity radius cannot be chosen --- it is not a tuning parameter.  One way to introduce some flexibility in the algorithm is to open the possibility of enlarging the search radius. In the context of an unweighted graph, the search radius is understood as the number of hops. With $m$ denoting the number of hops defining the local search area, which is now a tuning parameter, Graph Max Shift, initialized at node $i$, computes the following sequence of nodes:
\begin{equation}
%\label{graphmaxshift}
i_0 = i; \qquad i_t \in \argmax \{q_j : j \neigh_m i_{t-1}\}, \quad t\ge 1,
\end{equation}
where $i \neigh_m j$ if $i$ and $j$ are within $m$ hops away, or said differently, if the graph distance separating them is bounded by $m$. 

We anticipate the consistency result established in \thmref{consistency} to generalize in a straightforward manner under some appropriate conditions on $\eps$ and $m$. 
We note that choosing $m$ is nontrivial, as is typically the case for the choice of tuning parameters in unsupervised learning. We also note that, in practice, unless $m$ is quite small, the algorithm can be slow to run.

\subsection{Weighted graphs}
\label{sec:weighted}

We have focused on unweighted graphs, but the algorithm can be easily adapted to work with weighted graphs. Consider a graph with $n$ nodes, taken to be $[n] := \{1, \dots, n\}$ without loss of generality, and dissimilarity $\delta_{ij}$ between nodes $i$ and $j$, with $\delta_{ij} = \infty$ if $i$ and $j$ are not neighbors. A straightforward way to apply Graph Max Shift to such a graph is to apply a threshold so that it becomes an unweighted graph. Specifically, $i$ and $j$ would be redefined as neighbors if $\delta_{ij} \le h$, where $h$ is now a tuning parameter. The multi-hop variant described in \secref{multi-hop} is best implemented by defining the local search area around node $i$ as $\{j : \delta_{ij} \le r\}$, where $r \ge h$ is another tuning parameter. The correspondence with the number of hops $m$ introduced in \secref{multi-hop} is roughly $r \approx m h$ --- see \cite{arias2021estimation}, where this simple observation plays a central role.
The most appropriate situation here would be a weighted geometric graph, $\cG(\cY, \eps)$, where the weights are given by $\delta_{ij} = \|y_i-y_j\|$ if $\|y_i-y_j\| \le \eps$ and $=\infty$ otherwise. 

We expect the consistency result established in \thmref{consistency} to generalize to this setting under some appropriate conditions on $\eps$, $h$, and $r$, following the same line of arguments. 
Here too, choosing the tuning parameters $h$ and $r$ is challenging.

\section{Proofs}
\label{sec:proofs}

As we noted earlier on the in paper, the arguments supporting \thmref{consistency} are similar to those establishing the consistency of Max Shift in \cite{arias2025clustering}, but not exactly the same as the flat kernel \eqref{kernel} in use here is not twice differentiable or 2nd order, and the arguments in \cite{arias2025clustering} rely crucially on these properties being satisfied by the kernel function used to produce a density estimate. The reason for that is that the proof structure is laid out in \cite{arias2025clustering} to be applicable not just to Max Shift but to other methods such as Mean Shift. Because we focus on (Graph) Max Shift, we are able to tailor the arguments to the particular situation we are confronted with --- and actually simplify them in the process.

\subsection{Preliminaries}

\subsubsection{Density}
Recall that the density $f$ (on $\bbR^d$) is assumed to be compactly supported, twice continuously differentiable, and of Morse regularity. 
Define
\begin{equation}
\label{kappa}
\kappa_0 = \sup_{x \in \bbR^d} f(x), \qquad 
\kappa_1 = \sup_{x \in \bbR^d}\| \nabla f(x)\|, \qquad 
\kappa_2 = \sup_{x \in \bbR^d}\| \nabla^2 f(x)\|, 
\end{equation}
which are well-defined and finite by the fact that $f$ is assumed to be compactly supported and twice continuously differentiable.
Then, by Taylor expansion, the following inequalities hold for all $x, y$:
\begin{align}
\label{lipschitz}
    |f(x) - f(y)| \le \kappa_1 \|x-y\|; 
\end{align}
\begin{align}
\label{grad_lipschitz}
    \|\nabla f(x) - \nabla f(y)\| \le \kappa_2 \|x-y\|; 
\end{align}
\begin{align}
\label{taylor}
    \big|f(y) - f(x) - \nabla f(x)^\top (y-x)\big| \le \kappa_2 \|x-y\|^2. 
\end{align}
(In fact, $\kappa_2$ above can be replaced by $\kappa_2/2$, but using this looser inequality avoids carrying the $1/2$ factor around while it plays no essential role in our developments.)
Whenever $\nabla f(x)  \neq 0$, we denote the normalized gradient by 
\begin{equation}
\label{N}
N(x) := \frac{\nabla f(x)}{\| \nabla f(x)\|},
\end{equation}
which is differentiable when well-defined, with derivative equal to 
\begin{align}
\label{DN}
DN(x) = \frac{\nabla^2 f(x)}{\|\nabla f(x)\|} - \frac{\nabla f(x) \nabla f(x)^\top \nabla^2 f(x)}{\|\nabla f(x)\|^2}. 
\end{align}

\subsubsection{Upper level sets}
The upper $s$-level set of $f$ is defined as
\[\cU_s = \{x: f(x) \ge s\}.\]
We let $\cC_s(x)$ refer to the connected component of $\cU_s$ containing $x$.
%These connected components correspond to the clusters in Hartigan's definition of a cluster tree \cite{hartigan1975}, which is a perspective of clustering closely related to clustering via the gradient lines \cite{arias2023}. 

\begin{lemma} 
\label{lem:cluster_inclusion}
Let $x^*$ be a mode of $f$ and let $s^* = f(x^*)$. There exists $\lambda_1 \ge \lambda_0 > 0$ and $\delta > 0$ such that, for any $x \in \bar{B}(x^*, \delta)$, the eigenvalues of $\nabla^2 f(x)$ belong to $[-\lambda_1, -\lambda_0]$. Also, for any $s \in (s^* - \frac{1}{2}\lambda_0 \delta^2, s^*)$,
\begin{equation}
\label{cluster_inclusion}
\bar{B}\left(x^*, \sqrt{2 \lambda_1^{-1}(s^* - s)} \right) \subset \cC_s(x^*) \subset \bar{B}\left(x^*, \sqrt{2 \lambda_0^{-1}(s^* - s)}\right).
\end{equation}
\end{lemma}

\begin{proof} 
The first part of the statement relies on the fact that $\nabla^2 f(x^*)$ is negative definitive and that $\nabla^2 f$ is continuous. The second part is a restatement of Lemma~5.6 in \cite{arias2023}.
\end{proof}

For the kernel density estimator $\hat f = \hat f_\eps$, we will denote its upper $s$-level set and connected component containing $x$ as $\hat{\cU}_s$ and $\hat{\cC}_s(x)$, respectively, leaving $\eps$ implicit whenever possible.

\subsubsection{Approximations}
Recalling that $\hat f$ denotes the kernel density estimator that Graph Max Shift implicitly computes, which was defined in \eqref{kde}, set
\begin{equation}
\label{eta}
\eta := \sup_{x \in \bbR^d} \big|\hat f(x) - f(x)\big|.
\end{equation}

\begin{lemma}
\label{lem:hatC}
For any mode $x^*$ and any $s > 0$,  
\[\cC_{s + \eta}(x^*)  \subset \hat{\cC}_{s}(x^*) \subset \cC_{s - \eta}(x^*).\]
\end{lemma}

\begin{proof}
For any $x$, $\hat f(x) - \eta \le f(x) \le \hat f(x) + \eta$, so that, for any $s>0$, if $f(x) \ge s + \eta$ then $\hat f(x) \ge s$, and vice versa, if $\hat f(x) \ge s$ then $f(x) \ge s - \eta$. 
%Using the nested property of upper level sets, specifically, the fact that $\cU_s \subset \cU_t$ and $\hat\cU_s \subset \hat\cU_t$ for any $s > t$, 
We deduce that $\cU_{s+\eta} \subset \hat\cU_{s}$ and $\hat\cU_s \subset \cU_{s-\eta}$, and from that, we immediately get that any connected component of $\cU_{s+\eta}$ must be entirely within one of the connected components of $\hat\cU_{s}$ and that any connected component of $\hat\cU_s$ must be entirely within one of the connected components of $\cU_{s-\eta}$. Therefore, $\cC_{s + \eta}(x^*)$ is within a connected component of $\hat\cU_{s}$, which must be $\hat\cC_s(x^*)$ since this is the only one that contains $x^*$, and similarly, $\hat\cC_s(x^*)$ is within a connected component of $\cU_{s-\eta}$, which must be $\cC_{s - \eta}(x^*)$ for the same reason.
\end{proof}

%Henceforth, $a_n \lesssim b_n$ when $a_n = O(b_n)$ as $n \to \infty$, or put differently, when $a_n \le C b_n$ for some constant $C$ which does not depend on $n$.  

\begin{lemma}
\label{lem:eta}
%Let $Y_1,\dots,Y_n \iid f$. Then, for $\hat f$ defined in \eqref{kde}
In the present context, and with $\eps$ as in \thmref{consistency},
\[\eta = O_P\Big(\eps^2 + \textstyle\sqrt{\log(n)/n\eps^d}\Big).\]
\end{lemma}

Throughout, we use the notation $x = y \pm \eps$ when $\|x-y\| \le \eps$.

\begin{proof}
By the triangle inequality,
\begin{equation}
\label{density_estimation_proof1}
\eta 
\le \sup_{x\in \bbR^d} \big|\hat f(x) - \E[\hat f(x)]\big|
+ \sup_{x\in \bbR^d} \big|\E[\hat f(x)] - f(x)\big|.
\end{equation} 

For the stochastic term on the right-hand side of \eqref {density_estimation_proof1}, from \cite[Th~2.1]{gine2002rates} and because $\eps \gg (\log(n)/n)^{1/d}$ in \thmref{consistency}, we have
\[\sup_{x\in \bbR^d} \big|\hat f(x) - \E[\hat f(x)]\big| = O_P\big(\log(1/\eps)/n\eps^d\big)^{1/2}.\]
(The big-O term depends on $f, d$, but not on $n$. Recall that $\eps$ depends on $n$.) 

For the deterministic (bias) term on the right-hand side of \eqref {density_estimation_proof1}, take $x\in \bbR^d$. By the definition of the kernel density estimator, and using \eqref{taylor} along the way,
\begin{align}
\E[\hat f(x)] 
&= (v_d n \eps^d)^{-1} n \P(\|x-y_1\| \le \eps) \\
&= (v_d \eps^d)^{-1} \int_{B(x, \eps)} f(y) {\rm d}y \\
&= (v_d \eps^d)^{-1} \int_{B(x, \eps)} \Big(f(x) + \nabla f(x)^\top(y-x) \pm \kappa_2 \|y-x\|^2\Big) {\rm d}y \\
&= (v_d \eps^d)^{-1} \big(f(x)v_d \eps^d \pm \kappa_2 \eps^2 v_d \eps^d\big) \\
&= f(x) + O(\eps^2). 
\end{align}
We therefore have
\[\sup_{x\in \bbR^d} \big|\E[\hat f(x)] - f(x)\big|
= O(\eps^2).\]

We conclude by combining the bounds on the stochastic and deterministic terms, together with the fact that $\log(1/\eps) = O(\log n)$ when $\eps$ is as in the statement of \thmref{consistency}.
\end{proof}

For $s > 0$, set
\[\alpha_s := \max_{x \in \cU_{s}} \min_{y \in \cY} \|x - y\| = \max_{x \in \cU_{s}} \min_{i \in [n]} \|x - y_i\|,\]
which quantifies how dense the sample $\cY$ is in the set $\cU_{s}$. 
The following is standard. (See \cite[Lem 5.6]{arias2025clustering} for a detailed proof that follows the standard arguments and applies almost unchanged to the present context.)

\begin{lemma}
\label{lem:alpha}
In the present context, for any $s>0$,
\[\alpha_s = O_P(\log(n)/n)^{1/d}.\]
\end{lemma}

%\begin{proof}
%The proof is heavily based on Lemma 5.6 in \cite{arias2025clustering}.
%
%Let $\delta_s = \inf\{\|z-z_0\| : z \in \cU_s, f(z_0) = 0 \}$ be the smallest distance from a point in $\cU_s$ to a zero of $f$. Let $t_s := \min\{f(x): x \in \bar{B}(\cU_s, \delta_s/2)\}$. Note that $t_s>0$. 
%
%For any $\alpha < \delta_s$, let $z_1,\dots, z_{N_s(\alpha)} \in \cU_s$ be a minimal $(\alpha/2)$-covering of $\cU_s$. Note that $N_s(\alpha) \le \frac{C_(s)}{\alpha^d}$ (see for example, Lemma 2.7 in \cite{sen2018}).
%
%For any $z \in \cU_s$, the probability that no points in $\cY$ are in $B(z, \alpha/2)$ is $(1 - \P_f(B(z, \alpha/2)))^n$ where
%
%$$
%\P_f(B(z, \alpha/2)) = \int_{B(z, \alpha/2)}f(x) dx \ge v_d \left(\frac{\alpha}{2}\right)^d \inf_{x \in B(z, \alpha/2) } f(x) \ge  v_d \left(\frac{\alpha}{2}\right)^d t_s
%$$
%
%By a union bound, the probability that at least ball $B(z_k, \alpha/2)$ has no points in $\cY$ is bounded above by 
%
%$$
%N_s(\alpha)\left(1-  v_d \left(\frac{\alpha}{2}\right)^d t_s \right)^n \le N_s(\alpha) e^{-n v_d  \left(\frac{\alpha}{2}\right)^d t_s}
%$$
%
%This is an upper bound on $\P_f(\alpha_s \ge \alpha)$, and plugging in $\alpha = C_1(s)\left(\frac{\log{n}}{n}\right)^{1/d}$ with $C_1(s) = \frac{2^{d+1}}{v_d t_s}$ gives the desired result.
%\end{proof}

We recall that the graph distance between two nodes is the minimum number of hops, or length of the shortest path, to travel between the two nodes. 
Let $H_{ij}$ denote the graph distance between nodes $i$ and $j$. Under some circumstances, in the context of a geometric graph, the graph distance tracks the Euclidean distance. The following relies on a result established in \cite{arias2021estimation}.

\begin{lemma}
\label{lem:hops}
For any $s > 0$, there is a constant $A_0>0$ such that, if $\alpha_{s/2}/\eps \le 1/A_0$, 
\[0 \le H_{ij} \eps - \|y_i - y_j\| \le A_0 (\alpha_{s/2}/\eps) \|y_i - y_j\| + \eps,\]
whenever $y_i, y_j \in \cU_s$.
\end{lemma}

\begin{proof}
We note that the lower bound, which reads $\|y_i - y_j\| \le H_{ij} \eps$ is simply due to the fact that the straight line between $y_i$ and $y_j$ is the shortest curve between these points, and $H_{ij} \eps$ is the length of the polygonal line embedding the shortest path in the graph between nodes $i$ and $j$.

We therefore focus on establishing the upper bound.
Since $\cU_s$ is compact, it suffices to prove the result for $y_i, y_j \in \cU_s$ that are close enough. Let $t > 0$ be the distance between $\cU_s$ and $\cU_{s/2}^\comp$, so that $\bar B(x, t) \subset \cU_{s/2}$ for any $x \in \cU_s$. Now, consider $y_i, y_j \in \cU_s$ such that $\|y_i-y_j\| \le t$, so that $y_j \in \cS := \bar B(y_i, t)$. The set $\cS$ is convex and, since $\cS \subset \cU_{s/2}$, 
\[\max_{x \in \cS} \min_{y \in \cY} \|x-y\| \le \alpha_{s/2}.\]
Let $\tilde H_{ij}$ denote the graph distance between nodes $i$ and $j$ within the subgraph defined by the nodes $k$ such that $y_k \in \cS$. In such circumstances, if $\alpha_{s/2}/\eps \le 1/4$, \cite[Th~1]{arias2021estimation} yields
\[0 \le \tilde H_{ij} \eps - \|y_i - y_j\| \le 4 (\alpha_{s/2}/\eps) \|y_i - y_j\| + \eps.\]
And noting that $H_{ij} \le \tilde H_{ij}$, we are able to conclude.
\end{proof}

\subsection{Results about Max Shift}
\label{sec:max shift}

The next few results are about sequences computed by Max Shift based on $f$ with search radius $\eps$, namely, when started at $x_0$, 
\begin{align}
\label{max_shift}
x_{k+1} \in  \argmax\big\{f(x) : x \in \bar{B}(x_k, \eps)\big\}, \quad k\ge 0.
\end{align} 
%They are borrowed from \cite{arias2025clustering}.

%\begin{lemma}
%\label{lem:grad_x}
%Let $(x_k)$ be a sequence computed by Max Shift initialized at $x_0 \in \cU_s$ for some $s>0$. Assume that $\eps$ is small enough that any mode inside $\cC(x_0)$ is a maximum within radius $\eps$; and that it is smaller than the minimum separation between $\cC(x_0)$ and any other connected component of $\cU_{f(x_0)}$. Then for any $k$ except possibly the last one,
%\[
%x_{k+1} - x_{k} = \eps N(x_k) \pm \frac{C_4 r^{3/2}}{\| \nabla f(x_k) \|^{1/2}},
%\]
%whenever $\nabla f(x_k) \neq 0$.
%\end{lemma}
%
%\begin{proof}
%The first one is \cite[Lem 3.4]{arias2025clustering} and is analogous to \lemref{grad}.
%\end{proof}

%The first two results are variants of \lemref{step_size} and \lemref{grad}, respectively, given and proved further down, with $f$ in place of $\hat f$ and a search not restricted to $\cY$. 
The first one is essentially Lem~3.3 in \cite{arias2025clustering}, but not quite, while the second one is a straightforward consequence of the first.

\begin{lemma}%[Steps are approximately in the direction of the gradient]
\label{lem:grad_x} 
For the Max Shift sequence $(x_k)$ initialized at any $x_0$, 
\begin{equation} 
\label{grad_x}
    x_{k+1} - x_k = \|x_{k+1} - x_{k}\|N(x_k) \pm \frac{2 (2 \kappa_2)^{1/2} \eps^{3/2}}{\|\nabla f(x_k)\|^{1/2}}.
\end{equation}
\end{lemma}

\begin{proof}
Define $\tilde x_{k+1} := x_k + \eps N(x_k)$.
Using \eqref{taylor}, we derive
\begin{align}
f(\tilde x_{k+1}) 
&\ge f(x_k) + \nabla f(x_k)^\top (\tilde x_{k+1}-x_k) -\kappa_2 \|\tilde x_{k+1}-x_k\|^2 \\
&=  f(x_k) + \|\nabla f(x_k)\| \eps - \kappa_2 \eps^2.
\end{align}
Similarly, going in the other direction with starting point the fact that $\tilde x_{k+1} \in \bar{B}(x_k, \eps)$,
\[f(\tilde x_{k+1}) \le f(x_{k+1}) \le f(x_k) + \nabla f(x_k)^\top(x_{k+1} - x_k) + \kappa_2 \eps^2.\]
Combining these inequalities, we obtain 
\[
f(x_k) + \|\nabla f(x_k)\| \eps - \kappa_2 \eps^2
\le f(x_k) + \nabla f(x_k)^\top(x_{k+1} - x_k) + \kappa_2 \eps^2,\]
which after some simplifications yields
\begin{equation}
\label{inner_product_bound_x}
N(x_k)^\top(x_{k+1} - x_k) \ge \eps - \frac{2 \kappa_2 \eps^2}{\|\nabla f(x_k) \|}.
\end{equation}
Let $u_k := x_{k+1} - x_k$ and $\xi_k := 2 \kappa_2 \eps^2/\|\nabla f(x_k)\|$, so that \eqref{inner_product_bound_x} reads $N(x_k)^\top u_k \ge \eps - \xi_k.$
If $\eps < \xi_k$, we have
\begin{align}
\big\|u_k - \|u_k\|N(x_k)\big\|
\le 2 \|u_k\| 
\le 2 \eps
\le 2 (\eps \xi_k)^{1/2};
\end{align}
while if $\eps \ge \xi_k$, we have
\begin{align}
\big\|u_k - \|u_k\|N(x_k)\big\|^2 
&= 2\|u_k\|^2 - 2\|u_k\|N(x_k)^\top u_k \\
&\le 2\|u_k\|^2 - 2\|u_k\| \left(\eps - \xi_k\right) \\
&\le 2 \eps^2 - 2 \eps \left(\eps - \xi_k\right) \\
&= 2 \eps \xi_k.
\end{align}
Either way,
\begin{align}
u_k - \|u_k\|N(x_k) 
= \pm 2 (\eps \xi_k)^{1/2} 
= \frac{2 (2 \kappa_2)^{1/2} \eps^{3/2}}{\|\nabla f(x_k) \|^{1/2}},
\end{align}
so that \eqref{grad_x} holds.
\end{proof}

\begin{lemma}%[Step sizes are close to $\eps$] 
\label{lem:step_size_x}
For the Max Shift sequence $(x_k)$ initialized at any $x_0$, 
\begin{equation}
\label{step_size_x}
\|x_{k+1} - x_k \| \ge \eps - \frac{2 \kappa_2 \eps^2}{\|\nabla f(x_k) \|}.
\end{equation}
\end{lemma}

\begin{proof}
The bound follows from applying the  Cauchy--Schwarz inequality to the inner product on the left-hand side of  \eqref{inner_product_bound_x}.
\end{proof}

Note that combining \eqref{grad_x} and \eqref{step_size_x}, we have
\begin{align} 
x_{k+1} - x_k 
&= \eps N(x_k) \pm \left(\frac{2 (2 \kappa_2)^{1/2} \eps^{3/2}}{\|\nabla f(x_k)\|^{1/2}} + \frac{2 \kappa_2 \eps^2}{\|\nabla f(x_k) \|}\right) \\
&= \eps N(x_k) \pm \frac{2 (2 \kappa_2)^{1/2} \kappa_1^{1/2} \eps^{3/2} + 2 \kappa_2 \eps^2}{\|\nabla f(x_k) \|},
\label{Grad_x}
\end{align}
using the fact that $\|\nabla f(x_k)\| \le \kappa_1$.

The third result is distilled from the proof of Th 3.1--3.2 and arguments provided in Sec 3.2.1 in the same paper \cite{arias2025clustering}.

\begin{lemma}
\label{lem:con_x}
For any $p > 0$, there is $s > 0$ and a subset $\Omega \subset \cU_s$ with $\int_\Omega f \ge 1 - p$ such that any point in $\Omega$ is in the basin of attraction of a mode. Moreover, for any $\delta > 0$, there is $A_1> 0$ such that, if $\eps \le 1/A_1$, starting at any point $x_0 \in \Omega$ in the basin of attraction of some mode $x^*$, the sequence $(x_k)$ that Max Shift computes enters $B(x^*, \delta)$ in at most $A_1/\eps$ steps, while $\|\nabla f(x_k)\| \ge 1/A_1$ for all $k \ge 0$ such that $x_k \notin B(x^*, \delta)$.
\end{lemma}

With \eqref{Grad_x}, in the context of \lemref{con_x}, 
\begin{align} 
x_{k+1} - x_k 
%&= \eps N(x_k) \pm \frac{2 (2 \kappa_2)^{1/2} \kappa_1^{1/2} \eps^{3/2} + 2 \kappa_2 \eps^2}{\|\nabla f(x_k) \|} \\
&= \eps N(x_k) \pm C_0 \eps^{3/2},
\label{simple_Grad_x}
\end{align}
for all $k \ge 0$ such that $x_k \notin B(x^*, \delta)$, where $C_0 := 2 (2 \kappa_2)^{1/2} \kappa_1^{1/2} A_1 + 2 \kappa_2 A_1^{1/2}$.

\subsection{Proof of \thmref{consistency}}

We start by proving that, with arbitrarily high probability, a point which is in the basin of attraction of a mode is moved to a terminal point near the mode. In fact, we show that this happens in a uniform manner within the density upper level sets. From that, the proof of \thmref{consistency} will easily follow.

\begin{proposition}\label{prp:main}
%Suppose $\eps_n \gg \big(\frac{\log n}n\big)^{1/(d+2)}$. 
For any $p > 0$, there is $s>0$ and a subset $\Omega \subset \cU_s$ with $\int_\Omega f \ge 1-p$, and a constant $A$ such that any data point in $\Omega$ in the basin of attraction of a mode is moved by Graph Max Shift to a point within distance $A (\eps + (\alpha_{s/2} + \eta)/\eps)$ of the mode.
\end{proposition}

\begin{proof}[Proof of \prpref{main}]
Fix $p > 0$, let $s>0$ and $\Omega \subset \cU_s$ be as in \lemref{con_x}. Let $\alpha$ be short for $\alpha_{s/2}$.
Because $\cU_s$ is compact, it is enough to prove the result for $\eps$ and $q := (\alpha+\eta)/\eps$ small enough.
%For a constant $q > 0$ to be chosen small as we go, consider the event
%\begin{equation}
%\label{event}
%q \eps \ge \max\big\{\eta, \alpha\big\}.
%\end{equation}
%This event happens with probability tending to~1 because of the conditions imposed on $\eps$ in the statement of the theorem, as well as \lemref{eta} and \lemref{alpha}, which give the order of magnitude of $\eta$ and $\alpha$, respectively.
%We assume that $\eps$ is small enough that all the lemmas in \secref{max shift}, called upon below, apply and that \eqref{event} occurs with probability at least $1-p_1$.
%
%Note that we have already imposed a requirement on $q$ above, that $q < 1/2A_0$, so that $\alpha/\eps < 1/2A_0$. Below, each time there is a constraint on $q$, we assume it is small enough to satisfy the constraint. We just have to make sure that each constraint allows for $q>0$ fixed and that there are only a finitely many constraints.

Take any data point $y^0 \in \cY \cap \Omega$ as a starting point and let $(y^k)$ denote the sequence in \eqref{graph_max_shift} computed by Graph Max Shift; let $(x_k)$ denote the sequence computed by Max Shift based on $f$ as defined in \eqref{max_shift} with the same starting point $x_0 = y^0$. 
The following lemmas are proved later.

\begin{lemma}%[Steps are approximately in the direction of the gradient]
\label{lem:grad}
There is $A_2>0$ such that, provided $\kappa_1 \eps + 2\eta \le s/2$ and $\alpha_{s/2} \le \eps$, for the Graph Max Shift sequence $(y^k)$ initialized at any $y^0 \in \cU_{s}$, 
 \begin{equation} 
\label{grad}
    y^{k+1} - y^k = \|y^{k+1} - y^k\|N(y^k) \pm A_2  \frac{\eps^{1/2} (\eps^2 + \alpha_{s/2} + \eta)^{1/2}}{\|\nabla f(y^k)\|^{1/2}},
\end{equation}
and 
\begin{equation}
\label{step_size}
\|y^{k+1} - y^k \| \ge \eps - A_2\frac{(\eps^2 + \alpha_{s/2} + \eta)}{\|\nabla f(y^k) \|}.
\end{equation}
\end{lemma}

When $\eps$ and $q$ are small enough, the conditions of \lemref{grad} are met. In that case, combining \eqref{grad} and \eqref{step_size}, we have
\begin{align} 
y^{k+1} - y^k 
&= \eps N(y^k) \pm \left( A_2\frac{\eps^2 + \alpha + \eta}{\|\nabla f(y^k) \|} + A_2  \frac{\eps^{1/2} (\eps^2 + \alpha + \eta)^{1/2}}{\|\nabla f(y^k)\|^{1/2}} \right)\\
&= \eps N(y^k) \pm C_1 \frac{q^{1/2} \eps}{\|\nabla f(y^k) \|}, \label{Grad}
\end{align}
for some constant $C_1$, using the fact that $\|\nabla f(y^k) \| \le \kappa_1$. 
%The inequalities above are valid for all $k \ge 0$.

\begin{lemma}%[Local convergence]
\label{lem:local_convergence}
For any $s>0$, there is $\delta > 0$ and $A_3 > 0$ such that, for any mode $x^* \in \cU_s$, the Graph Max Shift sequence initialized at a point $y^0 \in \cY \cap B(x^*, \delta)$ converges to a point in $\cY$ within distance $A_3 (\eps + (\alpha_{s/2} + \eta)/\eps)$ of $x^*$.
\end{lemma}

When $\eps$ and $q$ are small enough, the conditions of \lemref{local_convergence} are met.
In that case, let $x^*$ be the mode to which $y^0$ is attracted by gradient ascent and let $\delta$ be as in \lemref{local_convergence}. 
By \lemref{local_convergence}, it suffices to show that 
\begin{equation}
\label{within ball}
\text{$y^k \in \bar{B}(x^*, \delta)$ for some $k$.}
\end{equation}

Assuming that $x_0 = y^0 \notin B(x^*, \delta/2)$, for otherwise our job is done, let $k_\#$ be such that $x_{k_\#} \in B(x^*, \delta/2)$ while $x_k \notin B(x^*, \delta/2)$ for $k < k_\#$, which is well-defined by \lemref{con_x}. The same lemma tells us that $k_\# \le A_1/\eps$. Assume that $q$ is small enough that $q \le \delta/2$, so that $\alpha \le \delta/2$. Then, by the triangle inequality, to establish \eqref{within ball} it is enough to show that
\begin{equation}
\label{consistency_proof1}
\|x_{k_\#} - y^{k_\#}\| 
\le \delta/2.
\end{equation}
Define $C_2 = 2 A_1 C_1$ and $C_3 = (2A_1+1) \kappa_2$, and assume $\eps$ and $q$ are small enough that
\begin{equation}
\label{complicated bound}
(C_0 \eps^{1/2} + C_2 q^{1/2}) C_3^{-1} (\exp(C_3 A_1) - 1) \le \frac1{2 A_1 \kappa_2} \wedge \frac\delta2.
\end{equation}
We show by recursion that 
\begin{equation}
\label{basic_recursion}
\|x_k - y^k\| 
\le (C_0 \eps^{1/2} + C_2 q^{1/2}) C_3^{-1} (\exp(C_3 \eps k) - 1),
\end{equation}
for all $k \le k_\#$.
We start the recursion by noting that \eqref{basic_recursion} is true at $k=0$ since $x_0 = y^0$.
Now, assume \eqref{basic_recursion} holds at some $0 \le k < k_\#$.
By the fact that $k_\# \le A_1/\eps$ and \eqref{complicated bound}, we also have 
\begin{equation}
\label{basic_recursion2}
\|x_k - y^k\| \le \frac1{2 A_1 \kappa_2}.
\end{equation}
Then, by \lemref{con_x} and a Taylor development,
\[
\|\nabla f(y^k)\| 
\ge \|\nabla f(x_k)\| - \kappa_2 \|x_k - y^k\|
\ge \frac1{A_1} - \kappa_2 \frac1{2 A_1 \kappa_2}
= \frac1{2 A_1}.
\]
Combining this with \eqref{Grad}, we find that 
\begin{align} 
y^{k+1} - y^k 
&= \eps N(y^k) \pm C_2 q^{1/2} \eps. \label{simple_Grad}
\end{align}
The inequalities \eqref{simple_Grad_x} and \eqref{simple_Grad}, together with the triangle inequality, imply
\begin{align}
\|x_{k+1} - y^{k+1}\| 
&\le \|x_{k} - y^k\| + \eps \|N(x_k) - N(y^k)\| + C_0 \eps^{3/2} + C_2 q^{1/2} \eps.
\end{align}
In view of \eqref{DN}, and the fact (justified by the same Taylor development) that $\|\nabla f(z)\| \ge 1/2 A_1$ for all $z$ on the line segment joining $x_k$ and $y^k$, 
\[\|N(x_k) - N(y^k)\| \le C_3 \|x_k - y^k\|,\]
implying that
\begin{align}
\|x_{k+1} - y^{k+1}\| 
&\le (1 + C_3 \eps) \|x_{k} - y^k\| + C_0 \eps^{3/2} + C_2 q^{1/2} \eps.
\end{align}
Using the recursion hypothesis \eqref{basic_recursion}, we find after some simplifications based on $e^a - 1 - a \ge 0$ for all $a$, that
\begin{equation}
\|x_{k+1} - y^{k+1}\| 
\le (C_0 \eps^{1/2} + C_2 q^{1/2}) C_3^{-1} (\exp(C_3 \eps (k+1)) - 1),
\end{equation}
so that the recursion carries on. (The arguments above underly a discrete form of Gr\"onwall's inequality.)
In particular, we have established that 
\begin{equation}
\|x_{k_\#} - y^{k_\#}\| 
\le (C_0 \eps^{1/2} + C_2 q^{1/2}) C_3^{-1} (\exp(C_3 A_1) - 1),
\end{equation}
which in view of \eqref{complicated bound} implies \eqref{consistency_proof1} .
\end{proof}

%%%%%%%%%%%%%
\begin{proof}[Proof of \lemref{grad}]
Let $\alpha$ be short for $\alpha_{s/2}$.
Let $\tilde y^{k+1}$ be the closest point in $\cY$ to $\tilde x_{k+1} := y^k + (\eps - \alpha)N(y^k)$ and assume $\alpha \le \eps$, so that $\|\tilde x_{k+1}-y^k\| = \eps - \alpha \le \eps$.
We have 
\begin{align}
    f(\tilde x_{k+1}) 
    &\ge  f(y^k) -\kappa_1 \eps \\ 
    &\ge \hat f(y^k) - \kappa_1 \eps - \eta \\
    &\ge \hat f(y^0) - \kappa_1 \eps - \eta \\ 
    &\ge f(y^0) - \kappa_1 \eps - 2\eta \\
    &\ge s - \kappa_1 \eps - 2\eta,
\end{align}
using \eqref{lipschitz}, \eqref{eta}, the fact that $\hat f(y^k)$ is increasing by definition, and \eqref{eta} again. Therefore, provided $\kappa_1 \eps + 2\eta \le s/2$, we have $f(\tilde x_{k+1}) \ge s/2$, so that $\tilde x_{k+1} \in \cU_{s/2}$, and by definition  of $\alpha$, $\|\tilde y^{k+1}-\tilde x_{k+1}\| \le \alpha$. Therefore, 
\[\|\tilde y^{k+1} - y^{k+1}\| \le \|\tilde y^{k+1}-\tilde x_{k+1}\| + \|\tilde x_{k+1}-y^{k+1}\| \le \eps-\alpha + \alpha \le \eps,\] by the triangle inequality, so that, by construction of $y^{k+1}$, 
\[\hat f(y^{k+1}) \ge \hat f(\tilde y^{k+1}),\] 
and continuing, using \eqref{eta}, \eqref{lipschitz}, and \eqref{taylor}, we derive
\begin{align}
\hat f(y^{k+1})
\ge \hat f(\tilde y^{k+1}) 
&\ge f(\tilde y^{k+1}) - \eta \\
&\ge f(\tilde x_{k+1}) - \kappa_1 \alpha - \eta \\
&\ge f(y^k) + \nabla f(y^k)^\top (\tilde x_{k+1}-y^k) -\kappa_2 \|\tilde x_{k+1}-y^k\|^2 -  \kappa_1 \alpha -\eta \\
&= f(y^k) + \|\nabla f(y^k)\| (\eps - \alpha) -\kappa_2(\eps - \alpha)^2 -  \kappa_1 \alpha -\eta \\
&\ge  f(y^k) + \|\nabla f(y^k)\| \eps - \kappa_2 \eps^2 - 2\kappa_1 \alpha - \eta.
\end{align}
In the other direction, using \eqref{eta} and \eqref{taylor}, 
\[\hat f(y^{k+1}) \le f(y^{k+1}) + \eta \le f(y^k) + \nabla f(y^k)^\top(y^{k+1} - y^k) + \kappa_2 \eps^2 + \eta.\]
Combining these inequalities gives
\[
f(y^k) + \|\nabla f(y^k)\| \eps - \kappa_2 \eps^2 - 2\kappa_1 \alpha - \eta
\le f(y^k) + \nabla f(y^k)^\top(y^{k+1} - y^k) + \kappa_2 \eps^2 + \eta,\]
which after some simplifications yields
\begin{equation}
\label{inner_product_bound}
N(y^k)^\top(y^{k+1} - y^k) \ge \eps - \frac{2\big(\kappa_2 \eps^2 + \kappa_1 \alpha + \eta\big)}{\|\nabla f(y^k) \|}.
\end{equation}
We then obtain \eqref{grad} from \eqref{inner_product_bound} exactly as we obtained \eqref{grad_x} from \eqref{inner_product_bound_x} in the proof of \lemref{grad_x}, but this time with $u_k := y^{k+1} - y^k$ and $\xi_k := 2 (\kappa_2 \eps^2 + \kappa_1 \alpha + \eta)/\|\nabla f(y^k)\|$. 

We then derive \eqref{step_size} by applying the Cauchy--Schwarz inequality to the inner product on the left-hand side of  \eqref{inner_product_bound}.
\end{proof}

%%%%%%%%%%%%%
%\begin{proof}[Proof of \lemref{step_size}]
%The upper bound is by construction of the sequence. The lower bound follows from applying the  Cauchy--Schwarz inequality to the inner product in the left-hand side of  \eqref{inner_product_bound}.
%\end{proof}

%%%%%%%%%%%%
\begin{proof}[Proof of \lemref{local_convergence}]
Because $\cU_s$ is compact, it is enough to prove the result when $\eps$ and $(\alpha_{s/2}+\eta)/\eps$ are small enough. And since there are finitely many modes in $\cU_s$, it is enough to consider a single mode $x^* \in \cU_s$. For that mode, let $s^* := f(x^*)$, so that $s^* \ge s$, and let $\delta_1, \lambda_0, \lambda_1$ be as defined in \lemref{cluster_inclusion}. Then, by a Taylor development of $\nabla f$, 
\begin{align}
\label{local_convergence_proof1}
\|\nabla f(x)\| \ge \lambda_0  \|x- x^*\|, \quad \text{for any $x \in B(x^*, \delta_1)$.}  
\end{align}
Let $s_1 := s^* - \frac{1}{2}\lambda_0 \delta_1^2$ and fix any $s_2 \in (s_1, s^*)$, for example, $s_2 = \frac12(s_1+s^*)$. Provided $\eta < s_2 - s_1$, invoking \lemref{cluster_inclusion}, we have
$$
\cC_{s_2 - \eta}(x^*) \subset B(x^*, \delta_1).
$$
Similarly, by the same result, letting $\delta_2 := \sqrt{(s^* - s_2)/\lambda_1}$, provided $\eta < \frac12 (s^* - s_2)$, we also have
$$
B(x^*, \delta_2)  \subset \cC_{s_2 + \eta}(x^*).  
$$
Therefore, if $\eta < \min\big\{ s_2 - s_1, \frac12 (s^* - s_2)\big\}$, using \lemref{hatC}, 
\begin{equation}
\label{local_convergence_proof2}
B(x^*, \delta_2)  \subset \cC_{s_2 + \eta}(x^*)  \subset \hat{\cC}_{s_2}(x^*) \subset \cC_{s_2 - \eta}(x^*) \subset B(x^*, \delta_1).
\end{equation}
Since we are assuming that $\eta < s_2 - s_1$, it is the case that $s_2-\eta > s_1$, so that $\cC_{s_2 - \eta}(x^*) \subset \cC_{s_1}(x^*)$. Assume that $\eps$ is strictly smaller than the minimum separation between the connected components of $\cU_{s_1}$ and consider Graph Max Shift initialized at some $y^0 \in B\left(x^*, \delta_2 \right)$. Note that $y^0 \in \hat\cC_{s_2}(x^*)$,  and by the fact that $(y^k)$ is hill-climbing with respect to $\hat f$, it must be the case that $(y^k) \subset \hat\cU_{s_2}$. In fact, since by \lemref{hatC} each connected component of $\hat\cU_{s_2}$ must be within a connected component of $\cU_{s_2 - \eta}$, and since $s_2 - \eta > s_1$, each connected component of $\cU_{s_2 - \eta}$ must be within a connected component of $\cU_{s_1}$, and these are separated by a distance exceeding $\eps$, and since the increments in $(y^k)$ are of length bounded from above by $\eps$, it must be the case that $(y^k) \subset \hat{\cC}_{s_2}(x^*)$, that is, the sequence must remain in the same connected component of $\hat\cU_{s_2}$.
Continuing, since there are a finite number of points in $\cY$, the sequence $(y^k)$ must converge in a finite number of steps. Let $y^\infty$ denote the limit, which by the fact that $(y^k) \subset \hat{\cC}_{s_2}(x^*)$ and \eqref{local_convergence_proof2} above must be in $B(x^*, \delta_1)$. Therefore, due to \eqref{local_convergence_proof1}, 
\begin{align}
\label{grad_lower_bound}
\|\nabla f(y^\infty) \| \ge \lambda_0 \|y^\infty - x^* \|.
\end{align}

We now take $s_1$ large enough that $s_1 \ge s^*/2$, and we let $\alpha$ be short for $\alpha_{s/2}$.
Define $z = y^\infty + (\eps - \alpha)N(y^\infty)$, and let $y$ be the closest point in $\cY$ to $z$. 
By \eqref{lipschitz} and the fact that $y^\infty \in \cC_{s_2 - \eta}$,
\begin{align}
f(z) 
\ge f(y^\infty) - \kappa_1 \eps
\ge s_2 - \eta - \kappa_1 \eps, 
\end{align}
so that $f(z) \ge s_1$ if $\eps$ and $\eta$ are small enough that $\kappa_1 \eps + \eta \le s_2 - s_1$, which we assume henceforth. Because of that, $\|y - z\| \le \alpha$, implying via the triangle inequality that $\|y-y^\infty\| \le \eps$, making $y$ a fair search point when the sequence is at $y^\infty$, and for that reason, it must be the case that $\hat f(y^\infty) \ge \hat f(y)$.
It is enough to prove the result when $\eps > \alpha$, and assuming this is the case, we derive 
\begin{align}
0 \ge \hat f(y) - \hat f(y^\infty) 
&\ge f(y) - f(y^\infty) - 2\eta \\
&\ge f(z) - f(y^\infty) - \kappa_1\alpha - 2\eta \\
&\ge \nabla f(y^\infty)^\top(z - y^\infty) - \kappa_2\|z - y^\infty\|^2 - \kappa_1\alpha - 2\eta \\
&= \|\nabla f(y^\infty)\| (\eps - \alpha) - \kappa_2(\eps - \alpha)^2 - \kappa_1\alpha - 2\eta \\
&\ge \|\nabla f(y^\infty)\| \eps - \kappa_2 \eps^2 - 2\alpha \kappa_1 - 2\eta.
\end{align}
Combining this with \eqref{grad_lower_bound}, we conclude that
\[
\|y^\infty - x^*\| \le \frac{\kappa_2 \eps^2 + 2\alpha \kappa_1 + 2\eta}{\lambda_0 \eps},
\]
and from this we conclude.
\end{proof}

\subsubsection{Proof of {\em (i)}}
\label{sec:(i)}
Let $\rho$ be the quantity on the LHS of \eqref{weakly consistent}. Given $p > 0$, let $s, \Omega, A$ be as in \prpref{main}, and consider the quantity
\begin{align}
\rho_\Omega := \frac{\#\big\{i < j: \widehat\bI(i,j) < \bI(i,j);\  y_i, y_j \in \Omega\big\}}{\binom{n}{2}},
% \longrightarrow 0,\ \text{in probability as $n \to \infty$}; 
\end{align}
where $\widehat\bI$ is the output partition of Graph Max Shift without a merging post-processing step.
By definition, $\rho_\Omega \le \rho$, and with $\#\Omega := \#\{i: y_i \in \Omega\}$, we also have 
\begin{align}
\rho - \rho_\Omega 
&\le \#\big\{i < j: y_i \notin \Omega \text{ or } y_j \notin \Omega\big\} \binom{n}{2}^{-1} \\
&= 1 - \#\Omega (\#\Omega - 1)/n(n-1) \\
&\to 1 - (\textstyle\int_\Omega f)^2 \\
&\le 1 - (1-p)^2,
\end{align}
in probability as $n\to\infty$, since $\#\Omega \sim \Bin(n, \int_\Omega f)$ and $\int_\Omega f \ge 1-p$.
By \lemref{eta} and our assumption on $\eps$ in {\em (i)}, we have $\eta = o_P(\eps)$; and by \lemref{alpha}, letting $\alpha = \alpha_{s/2}$, we also have $\alpha = o_P(\eps)$. Therefore, $A (\eps + (\alpha+\eta)/\eps) = o_P(1)$. Let $\delta$ be the minimum separation between two modes within $\cU_{s/2}$ --- which is positive since $\cU_{s/2}$ is compact and the modes are isolated --- and let $\cE$ be the event that $A (\eps + (\alpha+\eta)/\eps) < \delta/2$. Then $\P(\cE) \to 1$ as $n\to\infty$, and under $\cE$, $\rho_\Omega = 0$ by \prpref{main}. Bringing all together, we have argued that $\limsup \rho \le 1 - (1-p)^2$ in probability as $n\to\infty$, and given that $p$ can be chosen arbitrarily close to zero, we conclude that $\rho \to 0$ in probability.

\subsubsection{Proof of {\em (ii)}}
The arguments are, at first, completely parallel with those in \secref{(i)}. We redefine the notation used there to accommodate the setting of {\em (ii)}. 
Let $\rho$ be the quantity on the LHS of \eqref{consistent}. Given $p > 0$, let $s, \Omega, A$ be as in \prpref{main}, and consider the quantity
\begin{align}
\rho_\Omega := \frac{\#\big\{i < j: \widehat\bI(i,j) \ne \bI(i,j);\  y_i, y_j \in \Omega\big\}}{\binom{n}{2}},
% \longrightarrow 0,\ \text{in probability as $n \to \infty$}; 
\end{align}
where $\widehat\bI$ is the output partition of Graph Max Shift with a merging post-processing step with parameter $\tau \ge 1$.
By definition, $\rho_\Omega \le \rho$, and as in \secref{(i)}, \begin{align}
\rho - \rho_\Omega 
%&\le \#\big\{i < j: y_i \notin \Omega \text{ or } y_j \notin \Omega\big\} \binom{n}{2}^{-1} \\
&\le 1 - \#\Omega (\#\Omega - 1)/n(n-1) \\
&\to 1 - (\textstyle\int_\Omega f)^2 \\
&\le 1 - (1-p)^2,
\end{align}
in probability as $n\to\infty$.
By \lemref{eta} and our assumption on $\eps$ in {\em (ii)}, we have $\eta = o_P(\eps^2)$; and by \lemref{alpha}, letting $\alpha = \alpha_{s/2}$, we also have $\alpha = o_P(\eps^2)$. 
Therefore, $A (\eps + (\alpha+\eta)/\eps) = A (\eps + o_P(\eps))$. 
(Compare with \secref{(i)}.) 
Let $\cE$ be the event that $A (\eps + (\alpha+\eta)/\eps) \le 2 A \eps$, so that $\P(\cE) \to 1$ as $n\to\infty$.
First, eventually, $2 A \eps < \delta/2$, where $\delta$ is as in \secref{(i)}, and when this is the case, under $\cE$, as in \secref{(i)},
\begin{align}
\#\big\{i < j: \widehat\bI(i,j) < \bI(i,j);\  y_i, y_j \in \Omega\big\} = 0.
\end{align}
Further, suppose that $\tau \ge 4A$. By \prpref{main}, points that belong to the same basin of attraction within $\Omega$ are moved to terminal nodes within distance $2A\eps$ of the corresponding mode, and therefore within $4A\eps$ of each other by the triangle inequality, so that they are grouped together in post-processing, implying that we also have 
\begin{align}
\#\big\{i < j: \widehat\bI(i,j) > \bI(i,j);\  y_i, y_j \in \Omega\big\} = 0.
\end{align}
We thus have $\rho_\Omega = 0$, and we may conclude as we did in \secref{(i)}.

\small
\bibliographystyle{chicago}
\bibliography{ref}

\appendix
\label{appendix:Appendix}

\section{Additional Plots}

We depict the weak clustering error \eqref{weakly consistent} and clustering error \eqref{consistent} as a function of $\eps$ for the other Gaussian mixture models considered in \figref{mixtures}. Observe that the quadrimodal density for exhibits higher variance than the bimodal density, for example. 

\begin{figure}[htbp]
    \centering
    \begin{subfigure}{\textwidth}
        \centering
        \includegraphics[width=0.8\textwidth]{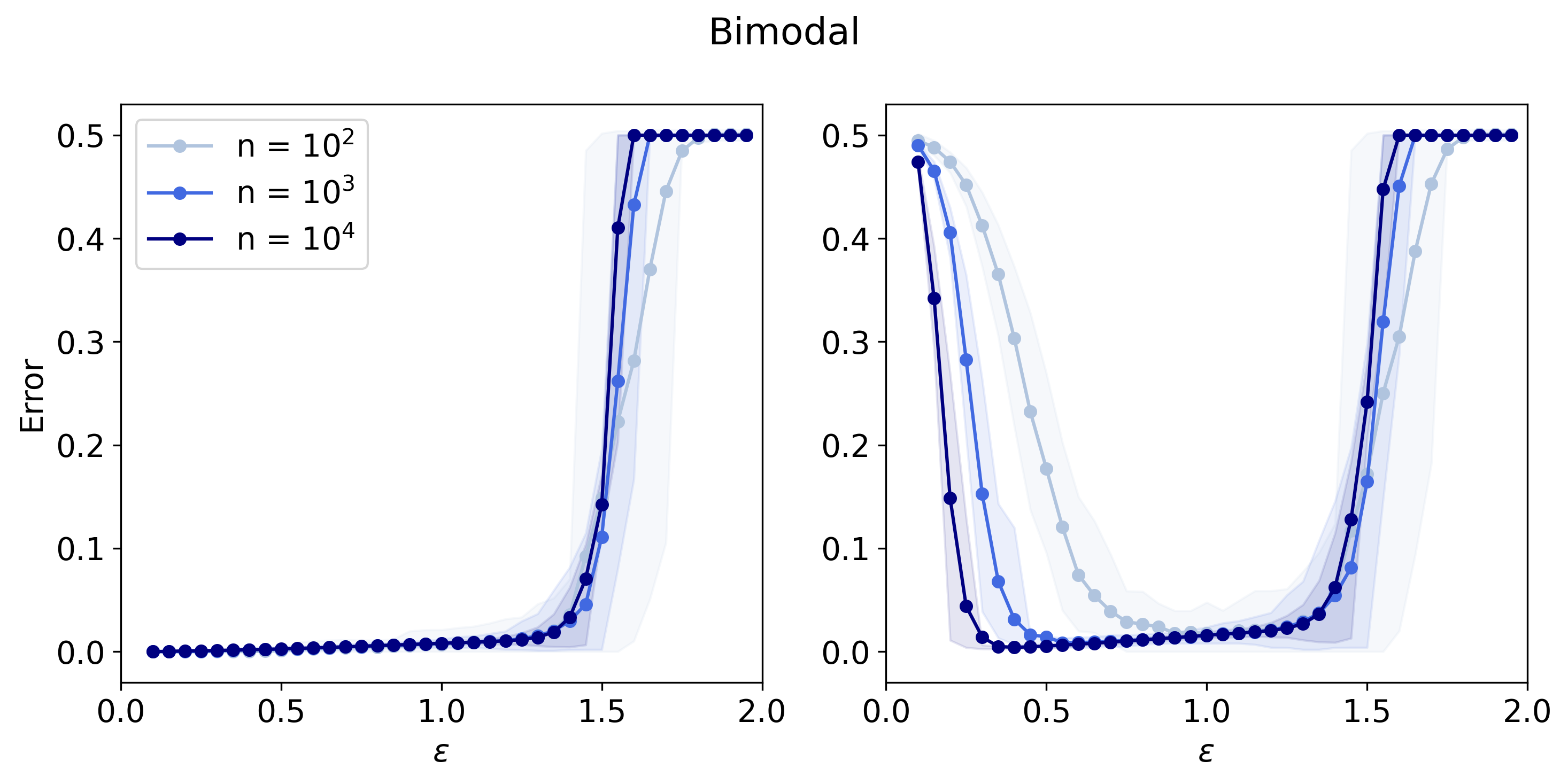}
    \end{subfigure}
    
    \begin{subfigure}{\textwidth}
        \centering
        \includegraphics[width=0.8\textwidth]{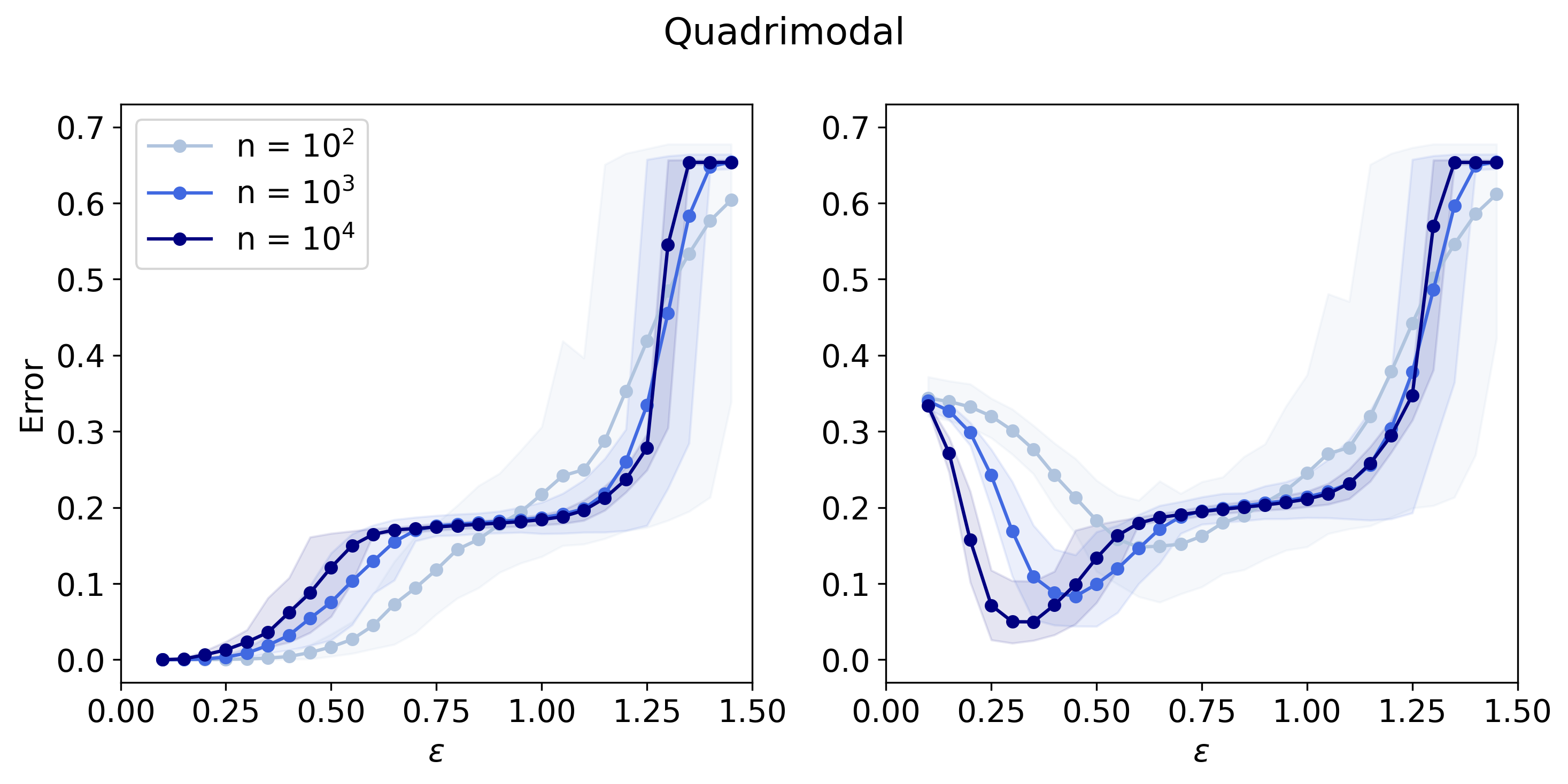}
    \end{subfigure}

    \begin{subfigure}{\textwidth}
        \centering
        \includegraphics[width=0.8\textwidth]{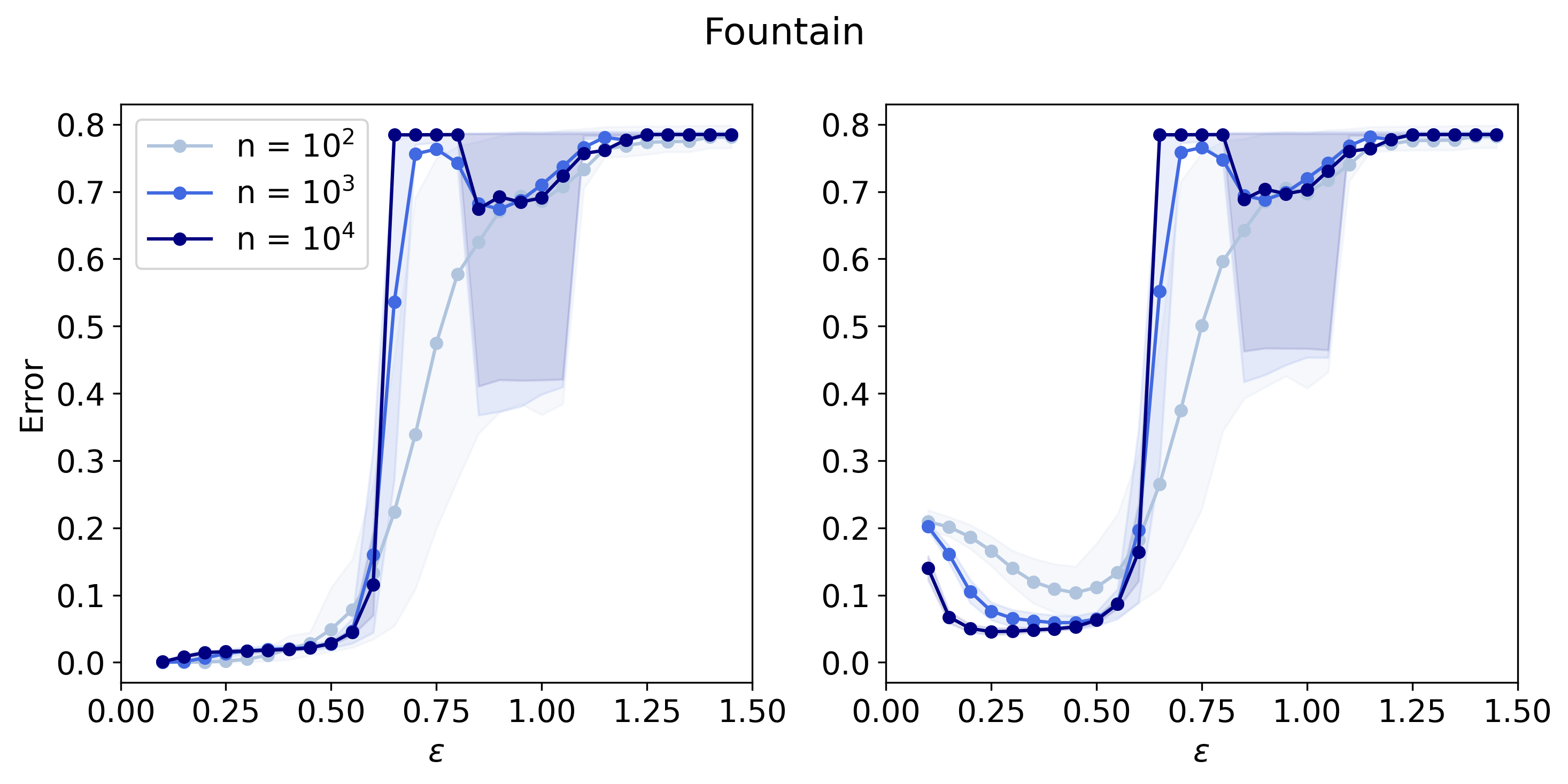}
    \end{subfigure}
    
    \caption{The weak clustering error (left) and the clustering error (right) of the other Gaussian mixtures from \figref{mixtures} when Graph Max Shift is applied to $\cG(\cY, \eps)$. The tuning parameter $\tau = 1$ for all experiments. Solid lines indicate the error averaged over $100$ simulations, and the shaded regions indicate empirical 80\% intervals.}
\end{figure}

We additionally include another bimodal Gaussian mixture. In this mixture, the separation between the two components is small so that the clustering task is more difficult. In \figref{gaussian_eps}, we  plot the clustering produced by Graph Max Shift with $\tau = 1$ over a range of values of $\eps$. Note that while the weak clustering error is small when $\eps = 0.035$ and $\eps = 0.06$, but the clusters are over-segmented in these cases. When $\eps$ is larger, both the weak clustering error and clustering error are large, as some points have crossed the basins of attraction. However, using the tuning parameter on the same sample, a small clustering error can be achieved. This is depicted in \figref{hop_tuning}, where $\tau$ is adjusted with $\eps$ fixed. When $\tau$ is too small the left cluster is split into multiple clusters, but by increasing $\tau = 2$, the nodes in the left component are merged together and most points are correctly clustered. When $\tau$ is too large, clusters in different basins of attraction are merged together.

\begin{figure}
    \centering
    \includegraphics[width=0.9\linewidth]{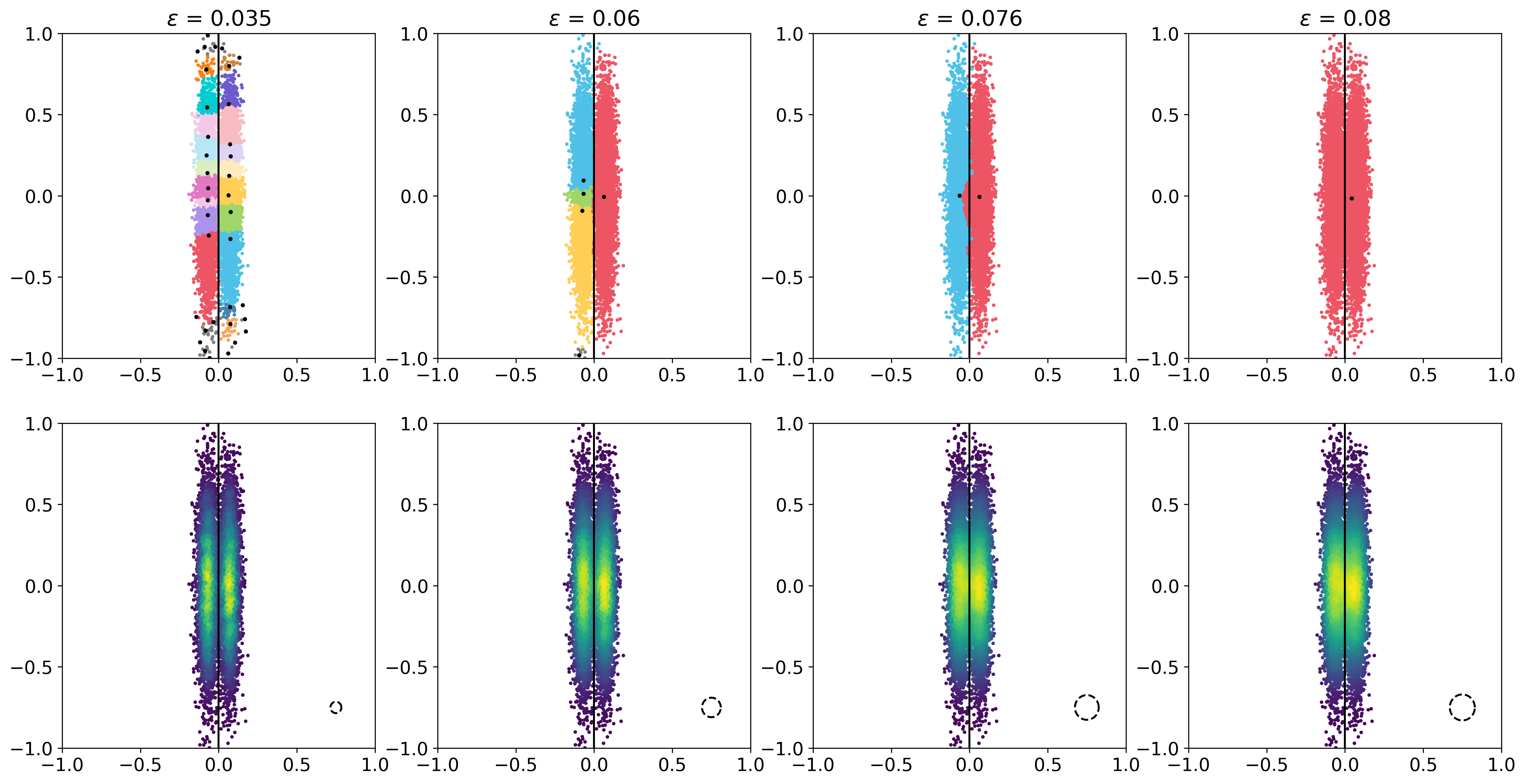}
    \caption{Graph Max Shift applied to $\cG(\cY;\eps)$ with data drawn from a bimodal Gaussian mixture. The top row shows the obtained clustering with the indicated $\eps$ and $\tau = 1$. The bottom row depicts the degree, which is proportional to the value of the density estimator implicitly computed. Additionally, each plot includes a ball of radius $\eps$ in the bottom right for reference.}
    \label{fig:gaussian_eps}
\end{figure}

\begin{figure}
    \centering
    \includegraphics[width=0.9\linewidth]{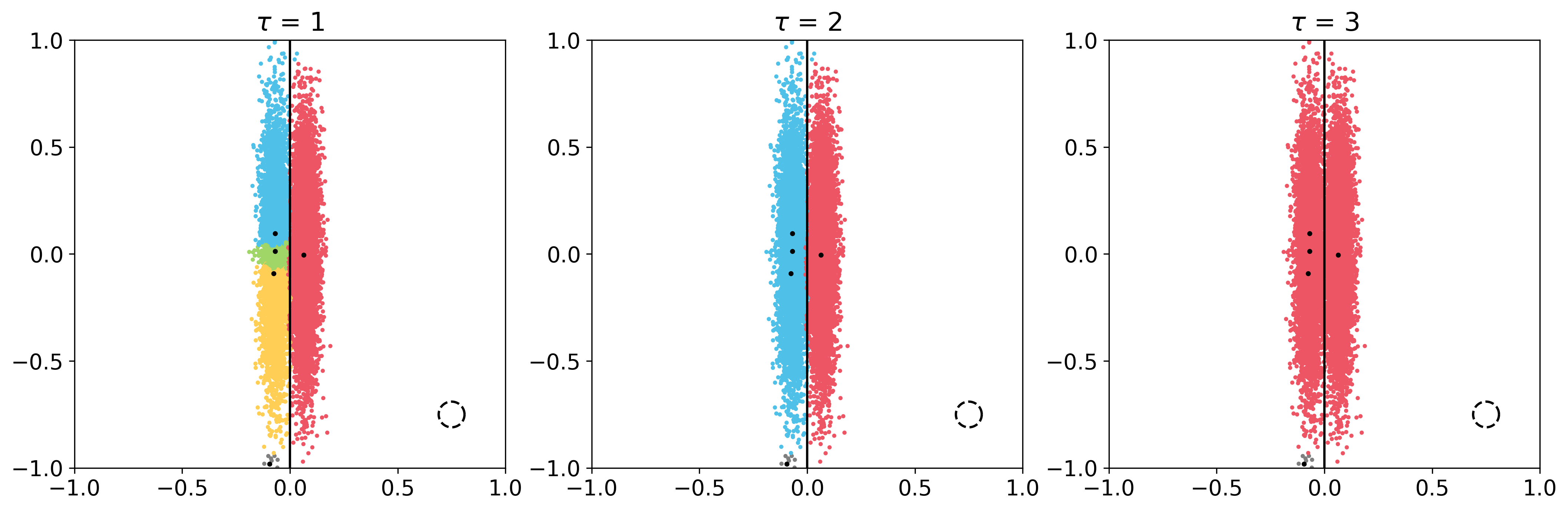}
    \caption{Graph Max Shift applied to $\cG(\cY;\eps)$ with data drawn from a bimodal Gaussian mixture. The obtained clustering with $\eps = 0.06$ fixed as $\tau$ is varied in $\{1, 2, 3\}$. The nodes associated with the clusters are plotted in black. Additionally, each plot includes a ball of radius $\eps$ in the bottom right for reference.}
    \label{fig:hop_tuning}
\end{figure}

%\subsection*{Acknowledgments} 
%This work was partially supported by the US National Science Foundation (DMS 1916071).

\end{document}